\documentclass{elsarticle}

\pdfoutput=1

\DeclareGraphicsExtensions{.pdf,.jpeg,.png}
\usepackage{amsmath}
\usepackage[noend]{algpseudocode} %loads algorithmicx
\usepackage{array}
  \usepackage[caption=false,font=footnotesize]{subfig}
\usepackage{algorithm} %bare_jrnl.tex (in IEEEtran) says do not use algorithm environment?
\usepackage{amssymb}
\usepackage{amsthm}
\usepackage{pgfplots}

\pgfplotsset{compat=1.13}

\newtheorem{theorem}{Theorem}%[section] %restart counter at new section
\newtheorem{definition}[theorem]{Definition}
\newtheorem{proposition}[theorem]{Proposition}

\algnewcommand\algorithmicswitch{\textbf{switch}}
\algnewcommand\algorithmiccase{\textbf{case}}
\algnewcommand\algorithmicassert{\texttt{assert}}
\algnewcommand\Assert[1]{\State \algorithmicassert(#1)}%
\algdef{SE}[SWITCH]{Switch}{EndSwitch}[1]{\algorithmicswitch\ #1\ \algorithmicdo}{\algorithmicend\ \algorithmicswitch}%
\algdef{SE}[CASE]{Case}{EndCase}[1]{\algorithmiccase\ #1}{\algorithmicend\ \algorithmiccase}%
\algtext*{EndSwitch}%
\algtext*{EndCase}%

\makeatletter
\newcommand{\StatexIndent}[1][3]{%
  \setlength\@tempdima{\algorithmicindent}%
  \Statex\hskip\dimexpr#1\@tempdima\relax}
\algdef{S}[WHILE]{WhileNoDo}[1]{\algorithmicwhile\ #1}%
\makeatother

\DeclareMathOperator*{\argmin}{arg\,min}
\DeclareMathOperator*{\argmax}{arg\,max}

\begin{document}

\title{Multi-Robot Patrolling with Sensing Idleness and Data Delay Objectives}

%\author{Michael~Shell,~\IEEEmembership{Member,~IEEE,}
%        John~Doe,~\IEEEmembership{Fellow,~OSA,}
%        and~Jane~Doe,~\IEEEmembership{Life~Fellow,~IEEE}% <-this % stops a space

%\thanks{M. Shell was with the Department
%of Electrical and Computer Engineering, Georgia Institute of Technology, Atlanta,

%elsearticle\author{J\"urgen~Scherer,
%elsearticle        Bernhard~Rinner$^{1}$%
%elsearticle\thanks{$^{1}$Both authors are with the Institute of Networked and Embedded Systems, Alpen-Adria-Universit\"at Klagenfurt, Austria,
%elsearticle{\tt\footnotesize \{juergen.scherer, bernhard.rinner\}@aau.at} }
%elsearticle}


\author[1]{J\"urgen Scherer\corref{cor1}}
\ead{juergen.scherer@aau.at}

\author[1]{Bernhard Rinner}
\ead{bernhard.rinner@aau.at}

\cortext[cor1]{Corresponding author}

\address[1]{Institute of Networked and Embedded Systems, University of Klagenfurt, Universit\"{a}tsstra\ss e 65-67, 9020 Klagenfurt, Austria}

%elsarticle\maketitle

\begin{abstract}
Multi-robot patrolling represents a fundamental problem for many monitoring and surveillance applications and has gained significant interest in recent years. In patrolling, mobile robots repeatedly travel through an environment, capture sensor data at certain sensing locations and deliver this data to the base station in a way that maximizes the changes of detection. Robots move on tours, exchange data when they meet with robots on neighboring tours and so eventually deliver data to the base station.

In this paper we jointly consider two important optimization criteria of multi-robot patrolling: (i) idleness, i.e. the time between consecutive visits of sensing locations, and (ii) delay, i.e. the time between capturing data at the sensing location and its arrival at the base station. We systematically investigate the effect of the robots' moving directions along their tours and the selection of meeting points for data exchange. We prove that the problem of determining the movement directions and meeting points such that the data delay is minimized is NP-hard. We propose heuristics and provide a simulation study which shows that the cooperative approach can outperform an uncooperative approach where every robot delivers the captured data individually to the base station.
\end{abstract}

%elsarticle\begin{IEEEkeywords}
%IEEE, IEEEtran, journal, \LaTeX, paper, template.
%elsarticleMulti-Robot Systems; Planning, Scheduling and Coordination; Cooperating Robots; Patrolling % Path Planning for Multiple Mobile Robots or Agents; delay minimization; tour graphs; minimum delay graphs
%elsarticle\end{IEEEkeywords}


%elsarticle\IEEEpeerreviewmaketitle

\maketitle	%elsearticle

\section{Introduction}
%elsarticle\IEEEPARstart{T}{he}
The interest in using mobile robot teams for surveillance and monitoring environments over a longer period of time has emerged with the advances in the fields of robotics, computation and communication. Examples for applications include disaster response~\cite{Erdelj2017}, \cite{Scherer2015}, \cite{Khan2018}, wildfire monitoring~\cite{Ghamry2016}, security tasks~\cite{Liu2013}, environmental monitoring~\cite{Rossi2016}, and exploration and mapping \cite{Masehian2017}, \cite{Santos2017}. The mobility of the robots extends their sensor coverage and allows areas to be monitored that cannot be covered efficiently with static sensors that also have to be deployed. The drawback is, that not all areas of the environment can be monitored at the same time. In certain scenarios it is not only important that certain locations of interest (referred to as sensing locations) get visited repeatedly, but also that the data captured by the sensors of the robots is transmitted to a base station in due time. This allows human mission operators to quickly assess a situation or that the collected data can be promptly processed for another purpose. We assume that the mobile robots and the base station are equipped with wireless transceivers to exchange data as well as sufficient memory to store the data. This enables the data to travel to the base station via multiple robots in a store-and-forward fashion. Two optimization criteria are essential to this multi-robot problem: \emph{idleness} and \emph{delay}. The first describes the time between consecutive visits at a sensing location, and the second describes the time between the capturing of data at a sensing location and its arrival at the base station. To optimize or constrain these criteria, coordinating the movement of the robots is necessary.

Monitoring of an environment over long time periods is related to the patrolling problem, where mobile robots continuously travel and sense the environment. While idleness is a common optimization criterion for the patrolling problem \cite{Machado2003}, \cite{Chevaleyre2004}, \cite{Lauri2014}, \cite{Portugal2014}, explicitly minimizing or constraining delay has experienced much less attention in literature. In contrast to most of the existing work, we focus on cooperative data transportation which eliminates the need of detours to the base station for every robot to deliver the data. This can improve idleness and allows to operate robots in environments where traveling to the base station is not possible for every robot (e.g. due to obstacles).

Depending on the representation of the environment, determining the optimal solution to a patrolling problem can be computationally demanding. Determining the optimal tours for minimal idleness on graphs for example is related to the traveling salesperson problem (TSP) \cite{Garey1979} and the k-TSP \cite{Frederickson1976} which are both NP-complete. To decouple the complexity of path planning from planning the coordinated data transport to the base station, we assume that closed tours for each robot are given. Scheduling robots on given tours considering some idleness criterion is a recurring problem in literature, e.g. \cite{Pasqualetti2012b}, \cite{Pasqualetti2012a}, \cite{Smith2012}, \cite{Chang2015}, and \cite{Kantaros2017}.

We consider the following patrolling scenarios. A set of closed tours which can have different lengths, one for each robot, is given and the robots are only allowed to move along these tours in a certain direction. Two robots can exchange the data captured on their tours when they are at certain positions on their tours (so-called meeting points) with the aim to transport the data to the base station lying on the tour of a particular robot (Figure~\ref{fig:scenarios}; see also Figure~\ref{fig:scenario_corridor} for an example of a corridor environment to be patrolled). The goal is to limit the maximum idleness to the lowest possible value determined by the tours and to minimize the delay over all sensing locations. This problem involves answering the following questions: (i) which robots should meet, (ii) in what direction should the robots move on their tours, and (iii) if there are more than one possible meeting points between tours, at which one should robots meet. Additionally, a schedule has to be determined which describes where the robots should wait for each other (in case one robot meets with more than one other robot on its tour). We show that all three questions are NP-hard and propose a heuristic for solving this problem. The first question is related to selecting a tour tree from a tour graph (explained in Section~\ref{sec:mdt}) and is therefore termed \emph{minimum delay tree} (MDT). The other two questions are related to extensions of MDT and are termted \emph{MDT with directions} (MDTD) and \emph{MDTD with meeting points} (MDTDM).

The contributions of this work can be summarized as follows: We formulate the MDT problem and its extensions and show that they are NP-hard. We describe an algorithm that efficiently constructs a solution that has the best possible idleness with the given tours. In case only the directions of the tour traversals have to be chosen (everything else is fixed), we describe a procedure that efficiently constructs a solution that also minimizes the delay. We propose two heuristics for MDTD which select the tour tree and the directions from a tour graph. Finally, we evaluate and compare the heuristics in experimental simulations.

The article is organized as follows: In Section~\ref{sec:relatedwork} we review the existing literature. In Section~\ref{sec:problem} we introduce some notation and formulate the idleness and delay criteria. In Section~\ref{sec:mdt} the MDT problem and its extensions are described and the heuristics for MDTD are presented in Section~\ref{sec:mdtd_heur}. In Section~\ref{sec:online} we describe the algorithm for the online execution once a solution for MDTD has been obtained. Section~\ref{sec:eval} describes the simulation results and Section~\ref{sec:conclusion} concludes the article.

\begin{figure}
	\centering
	\includegraphics[scale=0.23]{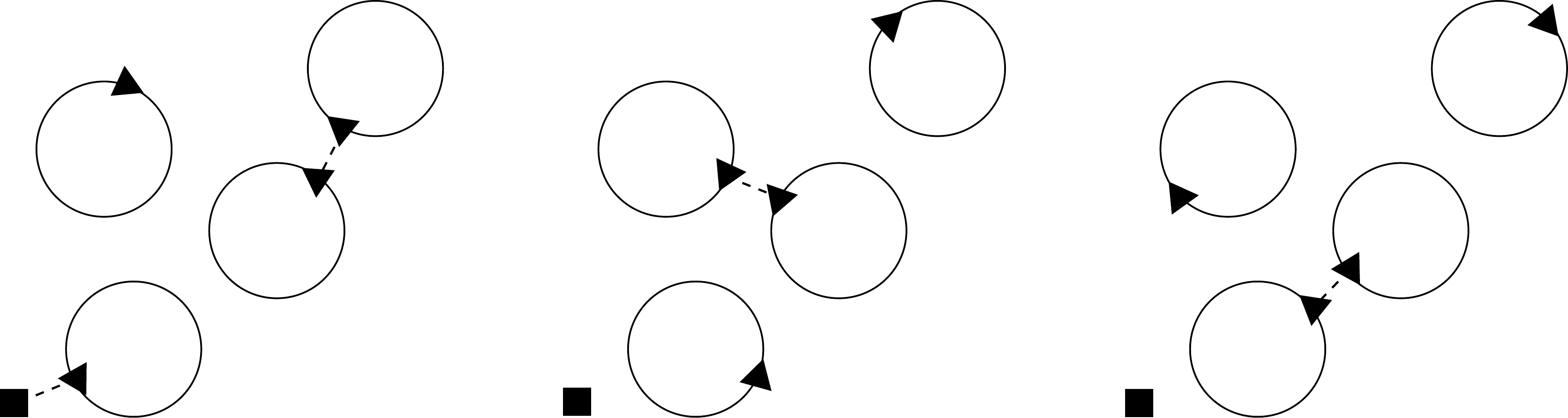}
	\caption{Example of a multi-robot patrolling scenario at three time instances (from left to right). Positions and directions of the robots are indicated by triangles, the base station is depicted as filled square. There is a dashed line between robots if they exchange data. The robots move along fixed tours (depicted as circles) and exchange data with robots on neighboring tours.}
	\label{fig:scenarios}
\end{figure}

\section{Related work}
\label{sec:relatedwork}

The multi-robot patrolling problem can be divided into the problem of determining paths in the environment and controlling and coordinating the robot movement along these paths. In \cite{Pasqualetti2012b} algorithms for the calculation of minimum idleness partitions for robots on a given chain, tree or cyclic graph are presented. In \cite{Pasqualetti2012a} a tour in the environment containing sensing locations with different priorities is calculated, and a control law that coordinates the robots on that tour is developed with the aim to minimize the weighted idleness. Coordinated patrolling accounting for leaving and joining robots on a linear perimeter with dynamic length is considered in \cite{Kingston2008}. Similarly, in \cite{Acevedo2013_IROS} robots travel along their partition on a linear perimeter and use local coordination with their neighbors to react to changes in perimeter length, number of robots and travel speed. In \cite{Smith2012} a velocity controller for robots following individual tours is developed. The goal is to limit uncertainty, which is growing in the environment at different rates. The problem of finding tours that meet idleness constraints of sensing locations is considered in \cite{Fargeas2013} and periodicity properties of these tours are investigated. In \cite{Nigam2012} the long term goal of minimizing the idleness is converted into a short horizon control law that selects the next sensing location that should be visited by a robot. In \cite{Chang2015} tour planning, dispatching robots on tours and controlling the speed to meet the revisit constraints of points of interest in a wireless sensor network setting with data mules is considered. %\cite{Mitchell2015}, \cite{Tokekar2015}

Maintaining connectivity is a prevalent requirement for multi-robot task planning \cite{Mosteo2009}, \cite{Grotli2012_GC}, \cite{Hollinger2012}, \cite{Ponda2012}, \cite{Flushing2017}. Persistent surveillance considering energy constraints and forcing persistent multi-hop connectivity to a base station are considered in \cite{Scherer2016} and \cite{Scherer2017}. The problem of minimizing the coverage time of an area with recurrent connectivity demands is presented in \cite{Anisi2010}. In \cite{Kantaros2017} robots travel back and forth along predefined paths between rendezvous points. A distributed controller determines the meeting times such that recurrent connectivity is guaranteed.

The MILP (mixed integer linear program) formulation and heuristics for the problem of finding a patrolling path for each robot with the goal to minimize the delay is presented in \cite{Banfi2015}. Each robot follows a path containing sensing locations and intermediate detours to communication sites where the data can be transmitted to the base station. A MILP formulation and a heuristic for a similar problem with task revisit constraints is presented in \cite{Manyam2017}. Patrolling considering the propagation of information among the robots is considered in \cite{Acevedo2013_ICUAS}. A decentralized algorithm maintains a grid shaped partition of the area where each robot is traveling a circular path within its subarea. Robots exchange data on the border of its subarea with each robot of the neighboring subareas, which minimizes the propagation time of information in this grid shaped partition. Table \ref{tab:relwork} summarizes the most relevant references for this work. 

\begin{table}
	\centering
	\begin{tabular}{l|c|c|c}
		Reference                & Persist. & Min. Del. & Connect. \\
		\hline
		\cite{Manyam2017}        & X     & X         & S \\
		\cite{Banfi2015}         &       & X         & S \\
		\cite{Flushing2017}      &       & X         & D \\
		\cite{Grotli2012_GC}     &       & o         & D \\
		\cite{Kantaros2017},
		\cite{Kantaros2019},
		\cite{Acevedo2013_ICUAS},
		\cite{Pasqualetti2012b}, 
		\cite{Kingston2008},
		\cite{Acevedo2013_IROS}  & X     &           & D \\
		\cite{Scherer2016},
		\cite{Scherer2017}       & X     &           & F \\
		\cite{Hollinger2012},
		\cite{Anisi2010}         &       &           & R \\
		\cite{Mosteo2009}        &       &           & F \\
		This work                & X     & X         & D
	\end{tabular}
\caption{Summary of related work. An 'X' in the column \textit{Persist.} indicates that the approach generates a solution for an infinite time horizon. An 'X' in \textit{Min. Del.} indicates that the work explicitly considers minimization of the delay and an 'o' that the delay is considered as constraint. The meaning of the letters in the \textit{Connect.} columns are: 'D' for delay tolerant (store-and-forward), 'S' for single hop (robots deliver data directly to the base station), 'R' for recurrent connectivity (all robots meet after a certain interval), and 'F' for full (persistent connectivity among all robots and the base station).}
\label{tab:relwork}
\end{table}

%\cite{Machado2003} \cite{Chevaleyre2004} \cite{Lauri2014} \cite{Portugal2014} \cite{Pasqualetti2012a} \cite{Smith2012} \cite{Chang2015} \cite{Pasqualetti2012b} \cite{Pasqualetti2012a} \cite{Kingston2008} \cite{Smith2012} \cite{Fargeas2013} \cite{Nigam2012}\cite{Chang2015} \cite{Ponda2012} \cite{Acevedo2013_IROS}

\section{Problem Formulation}
\label{sec:problem}

We assume that the tours are closed and can have different lengths, all points on a tour are sensing locations, and the tours contain predefined meeting points for the exchange of data with robots on neighboring tours. Robots can exchange collected data if they are at the meeting point that connects their tours at the same time. Since tours can have different lengths, it might be necessary for a robot to wait at a meeting point to meet its neighbor. All robots move with the same unit speed in a particular direction (either clockwise or counterclockwise) but every robot can stop at any point on its tour for an arbitrary amount of time.

More precisely, given is a set $V=\{1, \ldots, n\}$ of $n$ tours for $n$ robots. There is a one-to-one mapping between tours and robots, and we will use the same variable to identify a robot as well as the tour which it traverses. Every robot $v~\in~V$ moves along a tour in a particular direction $d_v$ and with unit speed, which is the same for all robots. With each tour a real number $l_v > 0$ is associated, which is the minimum time a robot can traverse the tour completely if there are no intermediate stops. Each point on a tour $v$ is from a set denoted $P(v)$ and has a coordinate in a local one-dimensional coordinate system which is determined by an origin on the tour and the direction of the tour. We assume that a subset of the points on a tour are sensing locations, and denote its set with $P_S(v)$ and $P^S := \bigcup_{v \in V} P_S(v)$. The vertices which contain sensing locations are denoted with $V_S$. The position of a meeting point between tours $w$ and $v$ is specified along the local coordinate system as $p_w^{meet}(v)$ on tour $w$, and as $p_v^{meet}(w)$ on tour $v$. The position of a robot $v$ at a certain instant $t$ on its tour is denoted by $p_v(t)$. Vertex $v_0$ identifies the tour which has a connection to the base station at point $p_{v_0}^{BS}$.

The tours are the vertices in a tour graph $G=(V, E, v_0, l_v, time_v, l_v^d)$, with an edge between two tours in $E$ if they are connected with a meeting point. Actually, a meeting point identifies two different points, one on each of the two tours it connects. The function $time_v(p,q,d):P(v)\times P(v)\times \{cw, ccw\} \rightarrow \mathbb{R}_{\geq 0}$\footnote{We denote the set of positive real numbers with $\mathbb{R}_{\geq 0}$} gives the minimum time for a robot to travel from meeting point $p$ to $q$ on tour $v$ in clockwise or counterclockwise direction $d$ (i.e. the distance between $p$ and $q$ under the unit speed assumption without intermediate stops). The function $l_v^d(p,d):P(v)\times \{cw, ccw\} \rightarrow \mathbb{R}_{\geq 0}$ returns the minimum time a sensing location from $p$ can be reached when moving in clockwise or counterclockwise direction. This function will allow the delay calculation after a robot starts its tour from a particular point $p\in P(v)$. In the following, we use the short notation $G=(V, E)$ for the tour graph. The tour graph is connected such that the data collected by the robots can reach a base station, which is connected to a particular tour.

To visit a sensing location $x \in P_S(r)$ at time $t$, a robot $r$ must be at the position of the sensing location at time $t$, i.e. $p_r(t)=x$. A robot $r_1$ visiting a sensing location $x \in P_S(r_1)$ at time $t$ captures and stores the observation data associated with the tuple $(x,t)$ in its local memory. The data is forwarded by robot $r_1$ to another robot $r_2$ at time $t' \geq t$ if $p_{r_1}(t')=x_1 \in P(r_1)$, $p_{r_2}(t')=x_2 \in P(r_2)$, and $x_1=p_{r_1}^{meet}(r_2)$ and $x_2=p_{r_2}^{meet}(r_1)$. Robot $r_2$ stores the data associated with the tuple $(x,t)$ and forwards it to any other robot it can communicate with at times $t'' \geq t'$. Finally, all data arrives at the base station $x_0 \in X$. We assume that a sensing location $x\in P^S$ is not considered to be visited when the robot stops at $x$ but when it again starts from $x$. In this way it is possible to decrease the delay by deferring the data generation to the latest possible time.

A patrolling strategy $\pi$ is a mapping from instants of time to points in $P_S(r)$ for every robot $r$ and describes when points should be visited by the robots. Two values are associated with a point $x \in P^S$, \emph{instantaneous idleness} and \emph{instantaneous delay}. The first describes the time a point remains unvisited, and the second describes the time between the capturing of observation data and its earliest arrival at the base station. The definition of the idleness criterion is adopted from \cite{Lauri2014} and is extended by the delay criterion. 

\begin{definition}[Instantaneous idleness, instantaneous worst idleness, worst idleness criterion \cite{Lauri2014}]
If the robots follow a strategy $\pi$, the instantaneous idleness $I_t^{\pi}(x) \in \mathbb{R}_{\geq 0}$ at time $t$ of point $x \in P^S$ is the elapsed duration since the last visit of $x$ by any robot. By convention, at initial time, $I_0^{\pi}(x) = 0$, for any strategy $\pi$ and any $x \in P^S$. The worst idleness criterion $WI^{\pi}$ is defined as 
\begin{equation}
	WI^{\pi}:=\limsup_{t \rightarrow +\infty}{WI_t^{\pi}},
\end{equation}
where $WI_t^{\pi} := \max_{x \in p_s}{I_t^{\pi}(x)}$ is the instantaneous worst idleness.
\end{definition}
Note that the definition of the instantaneous idleness considers the situation when a robot waits at a certain location, its instantaneous idleness stays zero as long as the robot is at that position.

\begin{definition}[Instantaneous visit delay, instantaneous delay, instantaneous worst delay, worst delay criterion]
If the robots follow a strategy $\pi$, the instantaneous visit delay $D_t^{\pi}(x, t', t'')$ at time $t$ of point $x \in P_s$ is the elapsed duration since a visit of any robot at point $x$ that happened at time $t'$ before the data arrives at the base station at time $t''$:
\begin{equation}
	D_t^{\pi}(x, t', t'') :=
		\begin{cases}
			t-t', & \text{if } t' \leq t \leq t''\\
			0, 		& \text{otherwise}
		\end{cases}
\end{equation}
%where $t''$ is the minimum time the data associated with the tuple $(x, t')$ arrives at the base station.
The instantaneous delay $D_t^{\pi}(x)$ of a point $x \in p_s$ is defined as $D_t^{\pi}(x) := \max_{t' \in T^V_x}{D_t^{\pi}(x, t', \min{T^R_{(x,t')}})}$, where $T^V_x$ is the set of points in time a visit at $x$ happens, and $T^R_{(x,t')}$ is the set of points in time the data associated with the tuple $(x,t')$ arrives at the base station. The worst delay criterion $WD^{\pi}$ is defined as 
\begin{equation}
	WD^{\pi}:=\limsup_{t \rightarrow +\infty}{WD_t^{\pi}},
\end{equation}
where $WD_t^{\pi} := \max_{x \in p_s}{D_t^{\pi}(x)}$ is the instantaneous worst delay.
\end{definition}

With this notation we define the MDT problem (and similar its extensions) as optimization problem
\begin{align}
\label{eq:opt_mdtd}
&\min_{\pi}{WD^{\pi}} \\
\text{s.t. } & WI^{\pi} \leq \max_{r \in R}{\{l_r\}} \\
& \pi \in \Pi_{MDT},
\end{align}
where $\Pi_{MDT}$ is the set of all solutions for the MDT instance, and $l_r$ is the minimal time for robot $r$ to completely traverse its tour. We will show that there exists a feasible solution for every instance of the MDT problem.

\section{Scheduling of robots on tours}
\label{sec:mdt}

In this section we discuss the problem of coordinating robots on predefined tours such that the idleness is bounded to the lowest possible value and the delay is minimized. Coordination comprises selecting the data exchange points where robots should meet and defining a travel direction for each robot on its tour. Selecting meeting points and directions determines the route of the captured data from a sensing location to the base station and has an effect on the delay. We will introduce the basic structure which describes the problem, the \textit{tour graph}. We will consider different variations of this problem: (i) selecting the directions when there is a minimal number of meeting points given, i.e. the tour graph is a tour tree, (ii) selecting a minimal number of meeting points when the directions are given, i.e. selecting a tour tree in a tour graph, (iii) selecting a minimal number of meeting points as well as directions, and (iv) selecting unique meeting points between tours, i.e. selecting a tree in a tour multi-graph. We will show that the latter three problems are NP-hard.

Figure~\ref{fig:example_simple_overview} shows an example of a tour graph which is a tree in this particular case (with $|V|=n=7$). If, like in this example, the tour graph is a tree and the directions are given, the path of data from an origin to the base station can be easily reconstructed. Assume that robot $1$ and robot $3$ have just met and robot $3$ has sent its collected data to robot $1$. After that, robot $1$ meets robot $2$, and robot $2$ continues collecting data as it moves along its tour to the meeting point with robot $1$ again. Here, after some time it meets robot $1$ again (possibly it has to wait for robot $1$) and sends the new data to robot $1$ which travels along its tour to the meeting point with robot $3$. Robot $3$ receives the data and moves to the meeting point with robot $5$. At the meeting point with robot $5$, it sends its own data and data received from robot $1$ (and robot $4$) to robot $5$. Finally, robot $5$ sends its own data and all data it received from robot $6$ and $3$ to the base station.

\begin{figure}
	\centering
	\begin{tabular}{cc}
	\subfloat[]{
		\includegraphics[scale=0.26]{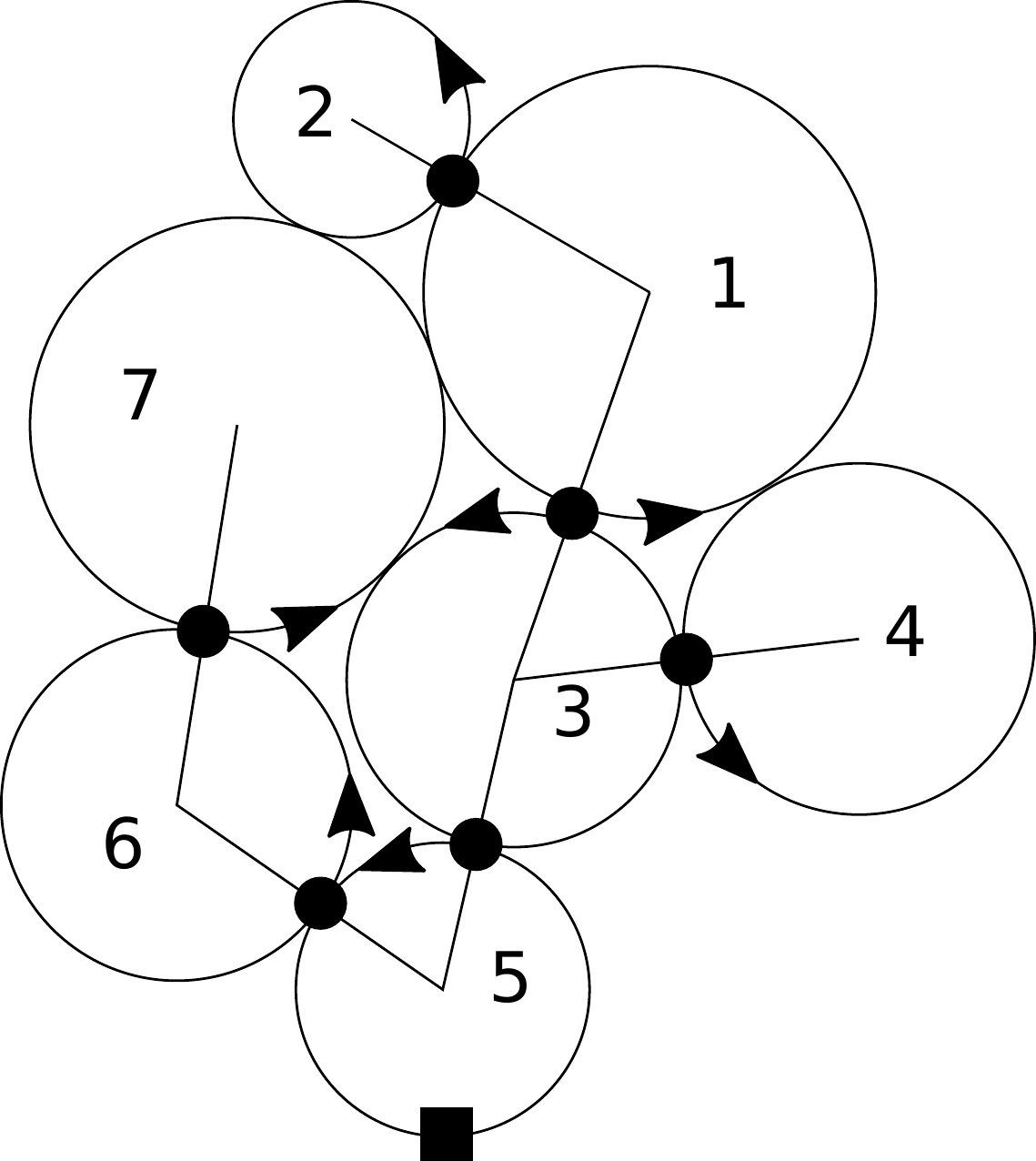}
		\label{fig:example_simple_overview}
	}
	&
	\subfloat[]{
		\includegraphics[scale=0.36]{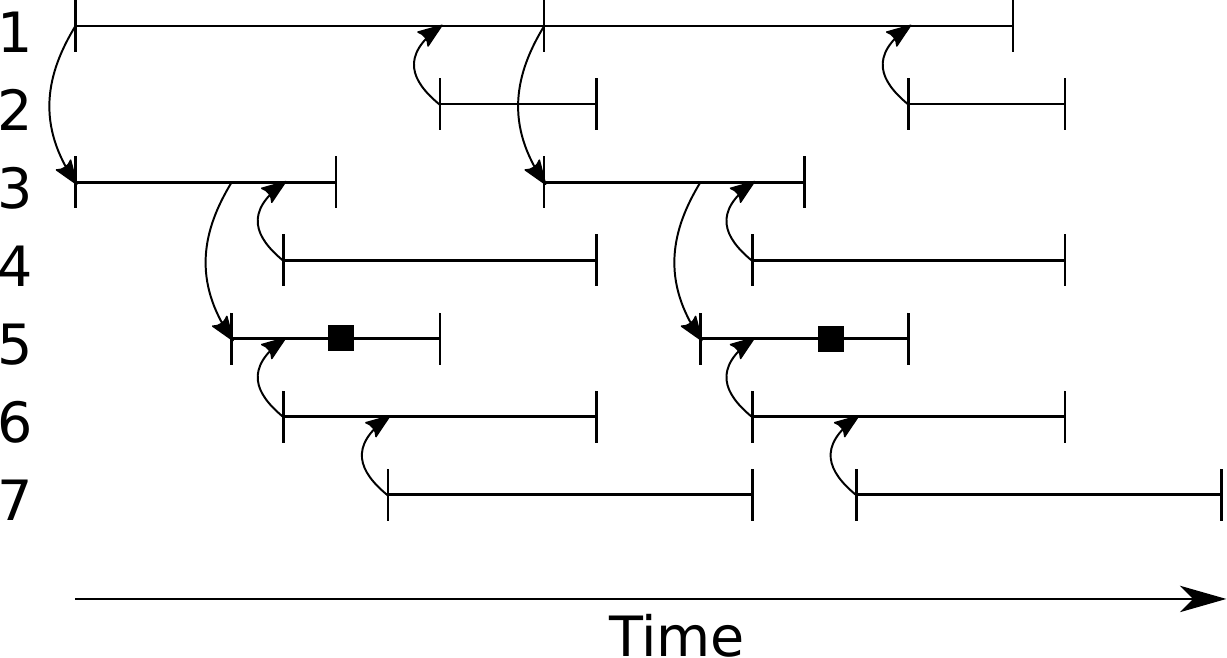}
		\label{fig:example_simple_timing}
	}
	\end{tabular}
	\caption{(a) Example of a tour graph which is a tree with $|V|=n=7$ and $P(v)=P_S(v)$ for all $v \in V$. The tours are depicted as circles, labeled by numbers, and furnished with arrows that show the movement direction of the robots. The solid small circles are the meeting points and the small solid rectangle is the base station. The straight lines indicate the edges of the tour graph. (b) A repeated schedule constructed from the graph in (a). A horizontal line denotes movement of a robot and the spacing between two horizontal lines indicates that a robot does not move. A vertical, curved arrows indicates that two robots meet and exchange data. The directions of the arrows indicate the  pathes the captured data travels towards the base station. The two solid rectangles show the position of the base station on tour 5.}
	\label{fig:example_simple_sched}
\end{figure}

\subsection{Selecting directions}
\label{subsec:dir}

We consider the situation where the tour graph forms a tree $T=(V, E)$ (the problem of selecting a tour tree in a tour graph is described in subsequent subsections), i.e. there is a minimal number of meetings points such that the tour graph is connected and the data from each sensing location can travel to the base station. We will first define a structure called \emph{schedule}, which contains the information for coordinating the robots on their tours. This information contains the start position of a robot on its tour, the direction of traversal, and the positions where a robot should stop and how long it should wait at a particular position.

\begin{definition}[Schedule]
Given a tour tree $T=(V, E)$, a schedule $\pi=\{\pi_1, \pi_2, $\allowbreak$\ldots, \pi_n\}$ is a set of tuples with $\pi_v=(p_v^{start}, d_v, wait_v)$. In a schedule a robot starts and stops at $p_v^{start}$, i.e., $p_v^{start}=p_v(0)=p_v(\tau_v)$, with $\tau_v:=l_v+\sum_{p\in P(v)} wait_v(p)$. The function $wait_v:P(v)\rightarrow \mathbb{R}_{\geq 0}$ defines the waiting times for points on tour $v$ (the waiting positions are the positions where the function returns values $> 0$). $d_v\in\{cw, ccw\}$ are the traversal directions (clockwise and counterclockwise). Additionally, $p_v(t)=p_v^{start} \; \forall t<0$ and $t>\tau_v$. $wait_v(p_v^{start})$ is the initial waiting time (possibly 0) starting at time 0. Every robot meets its neighbor, i.e., $\forall [v,w]\in E: \exists t \leq \tau: p_v(t)=p_v^{meet}(w)$ and $p_w(t)=p_w^{meet}(v)$, with $\tau:=\max_{i\in V}\{\tau_i\}$.
\end{definition}

The requirement that every robot meets its neighbors defined by the edges of the tree $T$ ensures that all collected data reach the base station (at least if the schedule is repeated). We define a \emph{repeated schedule} as an infinite horizon patrolling strategy $\pi^+$ that can be constructed from a schedule:

\begin{definition}[Repeated schedule]
A repeated schedule $\pi^+$ is a repetition of a schedule defined by $\pi^+=(\pi, \bar{v}, \gamma)$, where $\pi$ is a schedule, $\bar{v} \in V$, $\gamma \in \mathbb{R}_{\geq 0}$, such that
\begin{equation}
\label{eq:rep_sched}
\forall w\in V: \tau_w \leq \tau_{\bar{v}} + \Delta_{w\bar{v}} + \gamma,
\end{equation}
with $\Delta_{w\bar{v}}:=wait_w(p_w^{start})-wait_{\bar{v}}(p_{\bar{v}}^{start})$.
\end{definition}

The ``spacing'' between two repetitions of a schedule is defined by a robot $\bar{v}$ and $\gamma$, the time between the schedules of robot $\bar{v}$. Basically, inequality (\ref{eq:rep_sched}) states that each robot has to finish its tour before it can start again in the following repetition of the schedule. Obviously, the worst idleness $WI^{\pi^+} \geq L := \max_{v\in V}\{l_v\}$, the length of the largest tour traversed by a robot.

Figure~\ref{fig:example_simple_timing} shows two repetitions of a schedule of the tour tree in Figure~\ref{fig:example_simple_overview}. This repeated schedule can be defined by $(\pi, 1, 0)$ for example. Robot $1$ and $3$ start at their meeting point (described by the undirected edge $[1,3]$ in the tour graph) at the same time and robot $1$ moves without intermediate stops, whereas robot $3$ has to wait for robot $1$ when it finished its tour. As robot $1$ moves on its tour, it meets robot $2$, which is waiting for robot $1$ at the meeting point $[1,2]$. Robot $2$ starts to move, finishes its tour, and waits for robot $1$ to meet again at the meeting point. Note that the minimum worst idleness schedule in Figure~\ref{fig:example_simple_timing} does not necessarily minimize the worst delay. For example, when robot $2$ finished its tour, it has to wait for robot $1$ to transmit the captured data to it, which imposes a delay on the data of robot $2$ (robot $2$ could have postponed its start such that no new data is captured for a certain amount of time after it transmitted its data to robot~$1$).

\begin{proposition}
\label{prop:repsched}
A repeated schedule $\pi^+=(\pi, \bar{v}, 0)$, with $\bar{v} := \argmax_{v \in V}\{l_v\}$, can be constructed from any schedule $\pi$ with no intermediate waiting times, i.e. $wait_v(p) = 0, \forall p \in P(v)\setminus\{p_v^{start}\}$, $v \in V$. The worst idleness $WI^{\pi^+} = L$.	%l_{\bar{v}}
\end{proposition}
\begin{proof}
Since every robot follows its tour without intermediate stops and meets all neighbors, after a finite time $max_{v\in V}\{l_v+wait_v(p_v^{start})\}$, all robots have returned to the starting position. A repeated schedule must fulfill the inequality $\tau_w \leq \tau_{\bar{v}} + \Delta_{w\bar{v}} + \gamma$, $\forall w \in V$, which can be rewritten as $wait_w(p_w^{start}) + l_w \leq wait_{\bar{v}}(p_{\bar{v}}^{start}) + L + wait_w(p_w^{start}) - wait_{\bar{v}}(p_{\bar{v}}^{start})$ (which results in $l_w \leq L$).

Since the difference between the start times at the starting positions between $\bar{v}$ and any $w \in V$ in each repetition of the schedule $\pi$ is $\Delta_{w\bar{v}}$, and the difference between consecutive start times of $\bar{v}$ is $L$, also $WI^{\pi^+} = L$.
\end{proof}

Restricting the waiting times to the meeting point of $v$ with its parent to point $p_v^{start}$ has no negative impact on the delay since waiting at any other position on the tour cannot decrease the delay when the tour tree is given. Moreover, the data generation is deferred if $p_v^{start} \in P_S(v)$ due to the assumption described in Section~\ref{sec:problem}.

With a given tour tree, selecting the directions has an impact on the worst delay $WD$. Compared to the schedule with counterclockwise directions in the lower part of Figure~\ref{fig:example_simple_dir}, the schedule with clockwise directions in the upper part results in a lower delay.

Algorithm~\ref{alg:mindelay_sched} determines the schedule with directions of a given tour tree $T=(V, A)$. To identify the direction of an edge, the arc set $A$ is used where each edge from $E$ is directed towards the root node $v_0$, which contains the base station. In Line~\ref{line:min_delay_sched_callrec} the recursive procedure $\Call{rec}{v, u}$ is called. This function returns the maximum delay for a branch of a tree originating at tour $v$, including the path of the data on tour $v$ to its parent $u$ (the parent of $v$ is the unique node $u$ in an edge $(v, u)$). If $v$ is a leaf the direction is chosen that leads to smaller delay when robot $v$ starts at $p_v^{start}$. For tour $v$ the procedure tests which direction results in a smaller delay on tour $v$ given the maximum delays of the branches (Line~\ref{line:min_delay_sched_cw} and \ref{line:min_delay_sched_ccw}). The function $time_v(p,q,d)$ returns the time it takes to travel from point $p$ to point $q$ on a tour $v$ given the traversal direction $d$. Additionally, the procedure calculates the differences in the starting times and stores them in the variables $\Delta_{vw}$ (Line~\ref{line:min_delay_sched_delta}). These values are used to determine the starting times of the robots (Line~\ref{line:min_delay_sched_calcsched}).

The starting point $p_w^{start}$ of a robot $w$ is set to the meeting point $p_w^{meet}(v)$ with its successor $v$ in the traversal order (Line~\ref{line:min_delay_sched_start}). The only point where $w$ has to wait is the meeting point with $v$, and the waiting time is the sum of the waiting time of the successor $v$ and the previously calculated value $\Delta_{vw}$ (Line~\ref{line:min_delay_sched_calcsched}). This produces a schedule where $w$ is waiting for its successor $v$ and starts moving as soon it has met $v$, and follows the whole tour without intermediate stops. Finally, the wait times are shifted to be positive.

\begin{figure}
	\centering
	\begin{tabular}{cc}
		\subfloat[]{
			\label{fig:example_simple_cw}
			\includegraphics[scale=0.3]{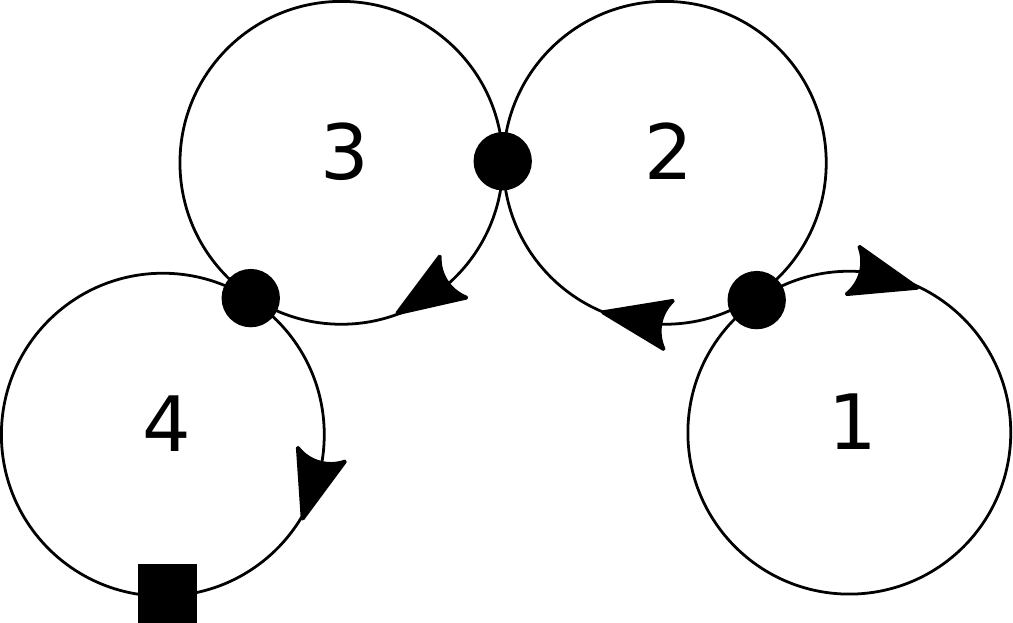}
		}
		&
		\subfloat[]{
			\label{fig:example_simple_cw_timing}
			\includegraphics[scale=0.43]{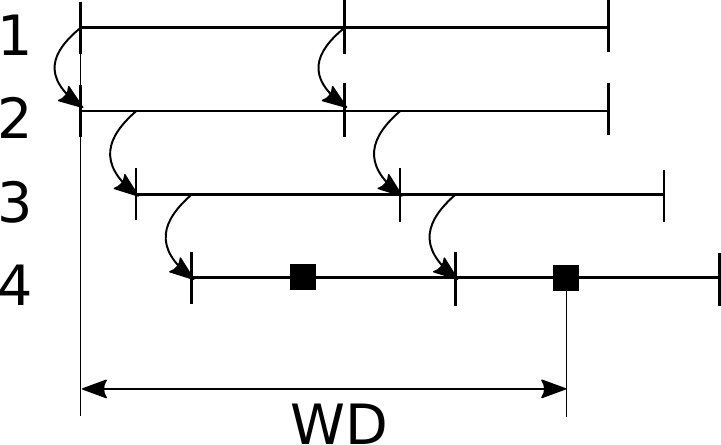}
		}
		\\
	%	\hfill
		\subfloat[]{
			\label{fig:example_simple_ccw}
			\includegraphics[scale=0.3]{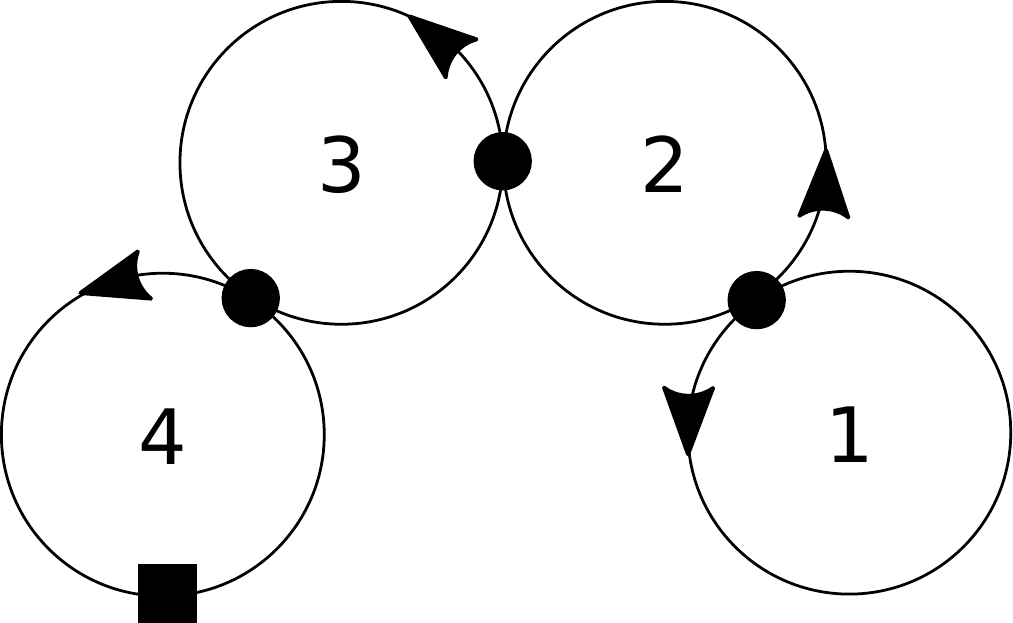}
		}
		&
		\subfloat[]{
			\label{fig:example_simple_ccw_timing}
			\includegraphics[scale=0.43]{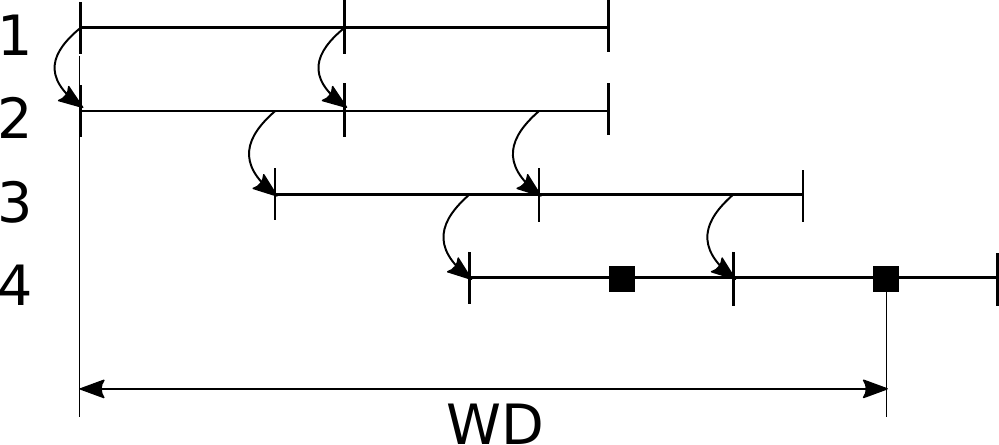}
		}
	\end{tabular}
	\caption{Tour tree with $P(v)=P_S(v)$ for all $v \in V$ where all robots move in clockwise (a) or in counterclockwise (c) directions. The worst delay for clockwise directions is smaller (b) than for counterclockwise directions (d).}
	\label{fig:example_simple_dir}
\end{figure}

\begin{algorithm}
	\caption{Minimum delay schedule}
	\label{alg:mindelay_sched}
	\small
	\begin{algorithmic}[1] % The number tells where the line numbering should start
		\Require
			\Statex \begin{flushleft} Tour tree $T=(V,A)$, base station vertex $v_0$, position of the base station (on tour $v_0$) $p_{v_0}^{BS}$, meeting positions $p_v^{meet}(\cdot)$, time of tours traversal $l_v$ $\forall v~\in~V$\end{flushleft}%to avoid spacing to fill line
		\Ensure
			\Statex \begin{flushleft} $wait_v(\cdot)$, start positions $p_v^{start}$, directions $d_v$ $\forall~v\in~V$ \end{flushleft}
		\State $\Call{rec}{v_0, null}$\label{line:min_delay_sched_callrec}
		\State $p_{v_0}^{start} \gets p_{v_0}^{BS}$
		\State $wait_{v_0}(p_{v_0}^{start}) \gets 0$
		\State $Q \gets <>$ \qquad //empty queue (LIFO)
		\State $Q.push(v_0)$
		\While{$Q \ne <>$}\label{line:min_delay_sched_while}
			\State $v \gets Q.pop()$
			\For{$w \in V$ with $(w,v) \in A$}
				\State $p_w^{start} \gets p_w^{meet}(v)$\label{line:min_delay_sched_start}
				\State $wait_w(p_w^{start}) \gets wait_v(p_v^{start}) + \Delta_{vw}$\label{line:min_delay_sched_calcsched}
				\State $Q.push(w)$
			\EndFor
		\EndWhile
		\State $mw \gets \min_{w}\{wait_w(p_w^{start})\}$
		\State $wait_v(p_v^{start}) \gets wait_v(p_v^{start}) - \min\{0, mw\}$ $\forall v \in V$

		\Statex
		\Procedure{rec}{$v$, $u$}
			\If{$v$ is leaf}
				\State $d_v \gets \argmin_{d \in \{cw, ccw\}}\{l_v - l_v^d(p_v^{start}, d)\}$
				\State return $\min_{d \in \{cw, ccw\}}\{l_v - l_v^d(p_v^{start}, d)\}$
			\EndIf
			\For{$(w,v) \in A$}
				\State $M_w \gets \Call{rec}{w, v}$
			\EndFor
			\State $M^{cw} \gets$ %\max_{(w,v) \in A}\{M_w + time_v(p_v^{meet}(u), p_v^{meet}(w), cw)\}$\label{line:min_delay_sched_cw}
			\StatexIndent[2]      $\max_{(w,v) \in A}\{M_w + time_v(p_v^{meet}(w), p_v^{meet}(u), cw)\}$\label{line:min_delay_sched_cw}
			\State $M^{ccw} \gets$ %\max_{(w,v) \in A}\{M_w + time_v(p_v^{meet}(u), p_v^{meet}(w), ccw)\}$\label{line:min_delay_sched_ccw}
			\StatexIndent[2]      $\max_{(w,v) \in A}\{M_w + time_v(p_v^{meet}(w), p_v^{meet}(u), ccw)\}$\label{line:min_delay_sched_ccw}
			%Done (last statement)? \State TODO: $l_v$ determines the worst case delay $WD$
			%Ignore: \State TODO: not every point on tour generates data ($l_v$ is not equivalent to $WD$)
			\If{$M^{cw} \leq M^{ccw}$}
				\State $d_v \gets cw$
			\Else
				\State $d_v \gets ccw$
			\EndIf
			\For{$(w,v) \in A$}\label{line:min_delay_sched_delta}
				\State $\Delta_{vw} \gets time_v(p_v^{meet}(u), p_v^{meet}(w), d_v) - l_w$
			\EndFor
			\State $M \gets \min\{M^{cw}, M^{ccw}\}$
			\State return $\max\{\min_{d \in \{cw, ccw\}}\{l_v - l_v^d(p_v^{start}, d_v)\}, M\}$
		\EndProcedure
	\end{algorithmic}
\end{algorithm}

\begin{proposition}
Algorithm~\ref{alg:mindelay_sched} produces a schedule $\pi$, from which a repeated schedule $\pi^+=(\pi, \bar{v}, 0)$, with $\bar{v}:=\argmax_{v\in V}\{l_v\}$, can be constructed. The worst idleness $WI^{\pi^+} = L$. Furthermore, $\pi^+$ minimizes the worst delay. %l_{\bar{v}}
\end{proposition}
\begin{proof}
In the loop in Line~\ref{line:min_delay_sched_while} the start times are chosen such that every robot $v$ meets its neighbors $w, \forall (w,v)\in A$ without intermediate stops. The worst idleness of $L$ follows from Proposition~\ref{prop:repsched}.

Now we show that the algorithm produces a schedule with minimum worst delay. Note that because of the chosen starting times, all captured data on a tour of a robot travels to the base station within the same schedule (assuming no repetition). When the minimum worst-case delays $M_w$ to node $v$ towards the base station in a call of $\Call{rec}{v, u}$ are known (which is certainly true if $w$ is a leaf), then $\max\{\min_{d \in \{cw, ccw\}}\{l_v - l_v^d(p_v^{start}, d_v)\}, M\}$ is also the minimum worst-case delay of all data including the data captured by $v$ until meeting position of $v$ with its parent towards the base station.

%Suppose, that the minimum worst case delays $M_w$ from the place of capturing to node $v$ towards the base station in a call of $\Call{rec}{v, p}$ are known, and assume, that $\max\{l_v, \min_{d_v \in \{cw, ccw\}}\{\max_{(w,v) \in A}\{M_w + time_v(p_v^{meet}(p), p_v^{meet}(w), d_v)\}\}\}$ is not the minimum worst case delay of all data including the data captured by $v$ until meeting position of $v$ with its parent towards the base station. Then, because the worst-case delays $M_w$ are known, $l_v$ does not determine the worst case delay. Then there must be an $M_w$ that determines the worst case delay ($w$ accounts for maximum value). Then, $M_w$ can be made smaller, which is a contradiction.
\end{proof}

\subsection{Selecting a tree in the tour graph}
\label{subsec:mdt}

Now we consider a tour graph $G=(V,E)$ instead of a tour tree. To show that the problem of determining a tree with a minimum delay schedule in a tour graph is NP-hard, we will formulate it as a decision problem \textit{d-MDT} and reduce the NP-complete problem 3SAT\footnote{An instance of 3SAT consists of a set $W$ of Boolean variables, and a set $C$ of clauses where each clause contains exactly three literals. The literals are of the form $x_i$ or $\overline{x_i}$ where $x_i \in W$. The question is, whether there is an assignment of values from $\{True, False\}$ to the variables such that in every clause at least one literal evaluates to $True$.} to it. We will assume that the directions of the tours are given and formulate d-MDT as follows. Given a tour graph with directions and the distances between the meeting points, and a bound $B$, the question is: is there a tree in the tour graph that admits a schedule with worst case delay of at most $B$? The optimization problem MDT cannot be easier than the decision problem, since a solution of the optimization problem also gives an answer to the decision problem. %of determining a tree which admits a minimum worst case delay

The construction of an d-MDT instance from an arbitrary 3SAT instance is shown by means of the example $\{c_1=\{x_1, x_2, x_3\}, c_2=\{\overline{x_1}, \overline{x_2}, x_4\}, c_3=\{x_2, \overline{x_3}, \overline{x_4}\}\}$ in Figure~\ref{fig:example_3sat_mdt}. In the reduction a vertex appears for each variable and each clause, and a meeting point connects a variable $x_i$ with a clause $c_j$ if the variable appears in the clause. The position of the meeting point on $x_i$ depends on whether the variable is complemented or not complemented in the clause. The basic idea is that for each clause $c_j$ an edge $(c_j, x_i)$ has to be selected such the data data from each $c_j$ can pass some $x_i$ with a low additional delay. This selection has the interpretation that the variable $x_i$ makes the clause evaluate to $True$. Since the result has to be a tree, a low additional delay for all clauses results in a satisfying assignment of the 3SAT instance. The details are described in the proof of the following proposition:

\begin{proposition}
\label{prop:mdt_np}
d-MDT is NP-hard.
\end{proposition}
\begin{proof}
Given an instance of 3SAT with variables $W=\{x_1, \ldots, x_a\}$, and clauses $C=\{c_1, \ldots, c_b\}$ a tour graph with the following $n=a+b+3$ vertices is constructed:
\begin{itemize}
	\item one vertex for every clause $c_i$
	\item one vertex for every variable $x_j$
	\item two vertices $x$ and $\overline{x}$
	\item a vertex $t$
\end{itemize}
The direction is set arbitrary and can be the same for all the tours, $P(v)=P_S(v)$ for all $v$, and the following meeting points between the tours are introduced:
\begin{itemize}
	\item On every $c_i$ there is a meeting point with every variable $x_j$ which appears as literal in $c_i$. The distances between the meeting points on $c_i$ is $2/3$.
	\item On every $x_i$ there are two meetings points with $x$ and $\overline{x}$ with distance $1$ between them on each side of the tour. The meeting points with the clauses $c_j$, where the variable $x_i$ appears, are grouped such that the distance to the meeting point with $x$ is $0$ and to the meeting point with $\overline{x}$ is $1$ if the variable appears as $x_i$ in $c_j$, and vice versa if the variable appears as $\overline{x_i}$ in $c_j$.
	\item On each of $x$ and $\overline{x}$ there is a meeting point with every variable $x_i$ with distance $0$ between them except for two meeting points, each with distance $1$ to the meeting point with $t$ (such that the distance between them is $2$ on the other side of the tour).
	\item On $t$ there are meeting points with $x$ and $\overline{x}$ with distance $0$ between them on one side of the tour and distance $1$ between each of them and the meeting point with the base station on the other side of the tour (such that the distance between them is $2$ on the other side).
\end{itemize}
The bound $B$ is set to $4$. Given a satisfying assignment for the variables $x_i$, the parent in the tree for a variable $x_i$ is $x$ if the variable is $True$ in the assignment, or $\overline{x}$ if the variable is $False$. The parent of a clause $c_j$ can be any $x_i$ that appears as satisfying literal in the clause. In this way the worst case delay, which is caused by the tours $c_j$, is $4$ (including the length of the tours). Note that the distances do not have to be $0$ and $1$, but sufficiently small and large, respectively. Based on these distances, the bound $B$ has to be set accordingly.

Next, we have to show that a tree with worst case delay of $4$ also determines a satisfying assignment for the 3SAT instance. We will do this by showing that the tree has to have a certain structure. First, both edges from $t$ to $x$ and $\overline{x}$ have to be in the tree. Otherwise, if e.g. $(t,x)$ is not chosen, data from $x$ to $\overline{x}$ has to pass some tour $x_j$ which leads to a delay of $5=2$ (length of tour $x$) + $1$ (on $x_j$) + $1$ (on $\overline{x}$) + $1$ (on $t$). Second, exactly one edge from any $x_j$ to either $x$ or $\overline{x}$ has to be in the tree. Choosing both edges results in a cycle containing $x_j$, $x$, $\overline{x}$, and $t$. If none of these edges is in the tree, the data has to travel along a path from $x_j$ to some $c_i$ to some $x_k$ and then to either $x$ or $\overline{x}$ which causes a delay of at least $4+2/3=2$ (length of tour $x_j$) + $2/3$ (on $c_i$) + $0$ (on $x_k$) + $1$ (on $x$ or $\overline{x}$) + $1$ (on $t$). Finally, for every $c_i$ exactly one edge to some $x_j$ has to be in the tree. Because of the arrangement of the meeting points on the tours $x_j$, choosing the edges for tours $c_i$ and the edges between $x_j$ and $x$ or $\overline{x}$ that admit a worst case delay of $4$, is equivalent to finding a satisfying assignment for the 3SAT instance.
\end{proof}

Consider the example in Figure~\ref{fig:example_3sat_mdt} again with the assignment $x_1=x_2=x_4=True$, and $x_3=False$. The parent of $x_1, x_2$, and $x_4$ is $x$, and the parent of $x_3$ is $\overline{x}$. The parent of $c_1$ can be either $x_1$ or $x_2$.

\begin{figure}
	\centering
	\includegraphics[scale=0.29]{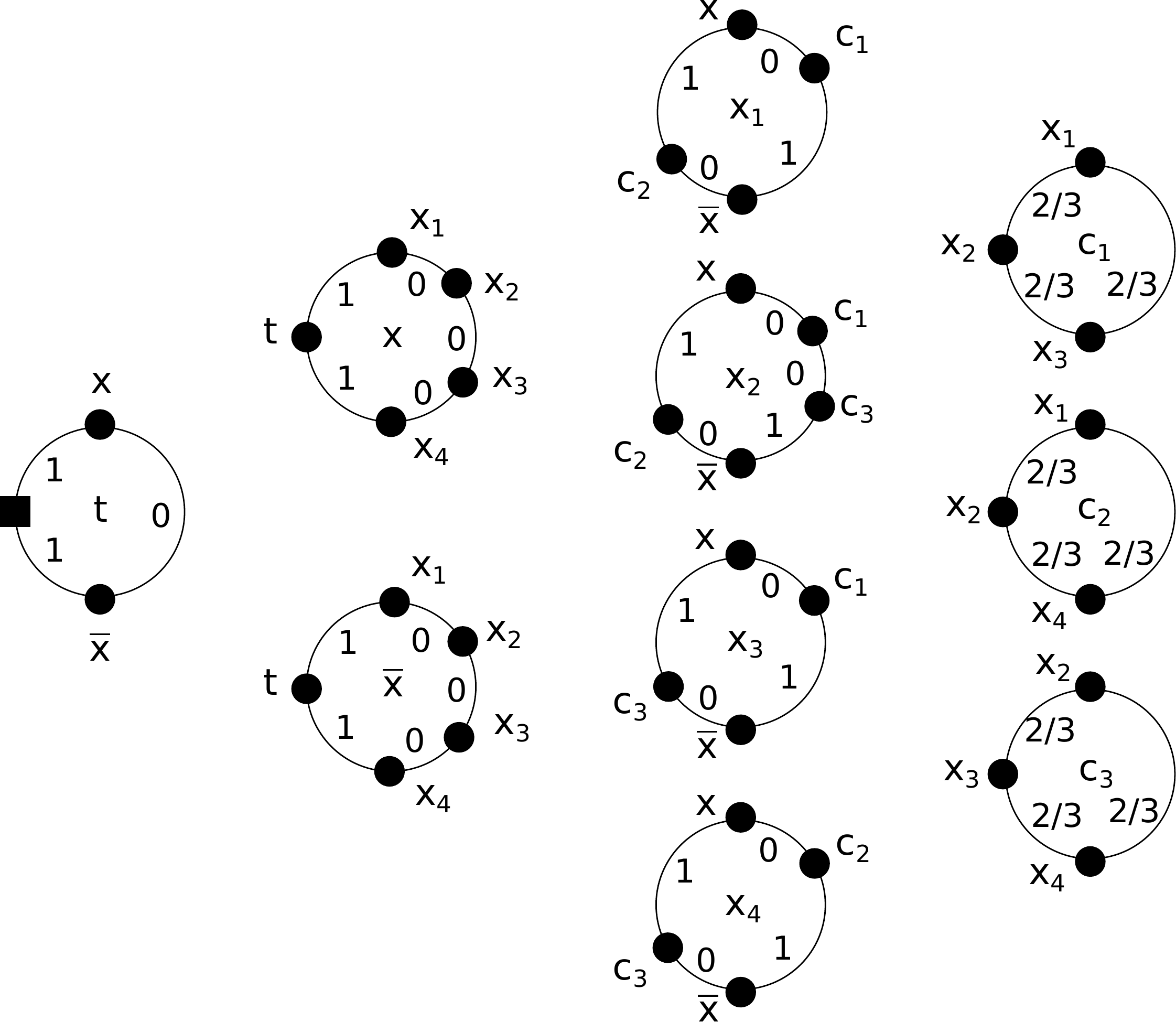}
	\caption{Example of a reduction from the 3SAT instance $\{c_1=\{x_1, x_2, x_3\}, c_2=\{\overline{x_1}, \overline{x_2}, x_4\}, c_3=\{x_2, \overline{x_3}, \overline{x_4}\}\}$ to d-MDT. The circles represent the tours (which do not touch for better readability), the connection between the tours are depicted with the named meeting points, and the direction is ccw for all tours.}
	\label{fig:example_3sat_mdt}
\end{figure}

\subsection{Selecting directions and meeting points}
\label{subsec:mdtd}

The problem minimum delay tree with directions (MDTD) is similar to MDT with the additional problem of finding the directions. We will show that the decision version d-MDTD is also NP-hard and present a heuristic algorithm for the problem.

\begin{proposition}
\label{prop:mdtd_np}
d-MDTD is NP-hard.
\end{proposition}
\begin{proof}
The proof is similar to the proof of Proposition~\ref{prop:mdt_np}. In addition to selecting a tree, the directions for traversing the tours have to be determined as well. The difference in the reduction is the arrangement of the meeting points on the tours for the variables $x_j$. The distance on the tour between a meeting point $c_i$ and $x$ is 0 if the variable $x_j$ does not appear as complement in clause $c_j$, and the distance between a meeting point $c_j$ and $\overline{x}$ is $0$ if the variable appears as complement. Figure~\ref{fig:example_3sat_mdtd} shows the construction of the reduction for the same example as in Section~\ref{subsec:mdt}.

The direction of the tours except for the tours corresponding to variables can be set arbitrary. If in an assignment a variable $x_j=True$, then the direction of the corresponding tour is counterclockwise, and clockwise otherwise. Therefore, a satisfying assignment admits a tree with worst case delay of $4$.

Again, to show that a tree with worst case delay of $4$ also determines a satisfying assignment for the 3SAT instance, we will show that the tree has to have a certain structure. First, both edges from $t$ to $x$ and $\overline{x}$ have to be in the tree. Otherwise, if e.g. $(t,x)$ is not chosen, the shortest possible path for data from $x$ to $\overline{x}$ has to pass some tour $x_j$ and some $c_i$ and some $x_k$ which leads to a delay of $4+2/3=2$ (length of tour $x$) + $0$ (on $x_j$) + $2/3$ (on $c_i$) + $0$ (on $x_k$) + $1$ (on $\overline{x}$) + $1$ (on $t$). Second, exactly one edge from any $x_j$ to either $x$ or $\overline{x}$ has to be in the tree. Choosing both edges results in a cycle containing $x_j$, $x$, $\overline{x}$, and $t$. If none of these edges are in the tree, the data has to travel along a path from $x_j$ to some $c_i$ to some $x_k$ and then to either $x$ or $\overline{x}$ which causes a delay of at least $4+2/3=2$ (length of tour $x_j$) + $2/3$ (on $c_i$) + $0$ (on $x_k$) + $1$ (on $x$ or $\overline{x}$) + $1$ (on $t$). Finally, for every $c_i$ exactly one edge to some $x_j$ has to be in the tree. To limit the delay for data from tours $c_i$ to $4$, the direction for the tours $x_j$ have to be chosen accordingly. This is only possible if the 3SAT instance has a satisfying assignment.
\end{proof}

\begin{figure}
	\centering
	\includegraphics[scale=0.29]{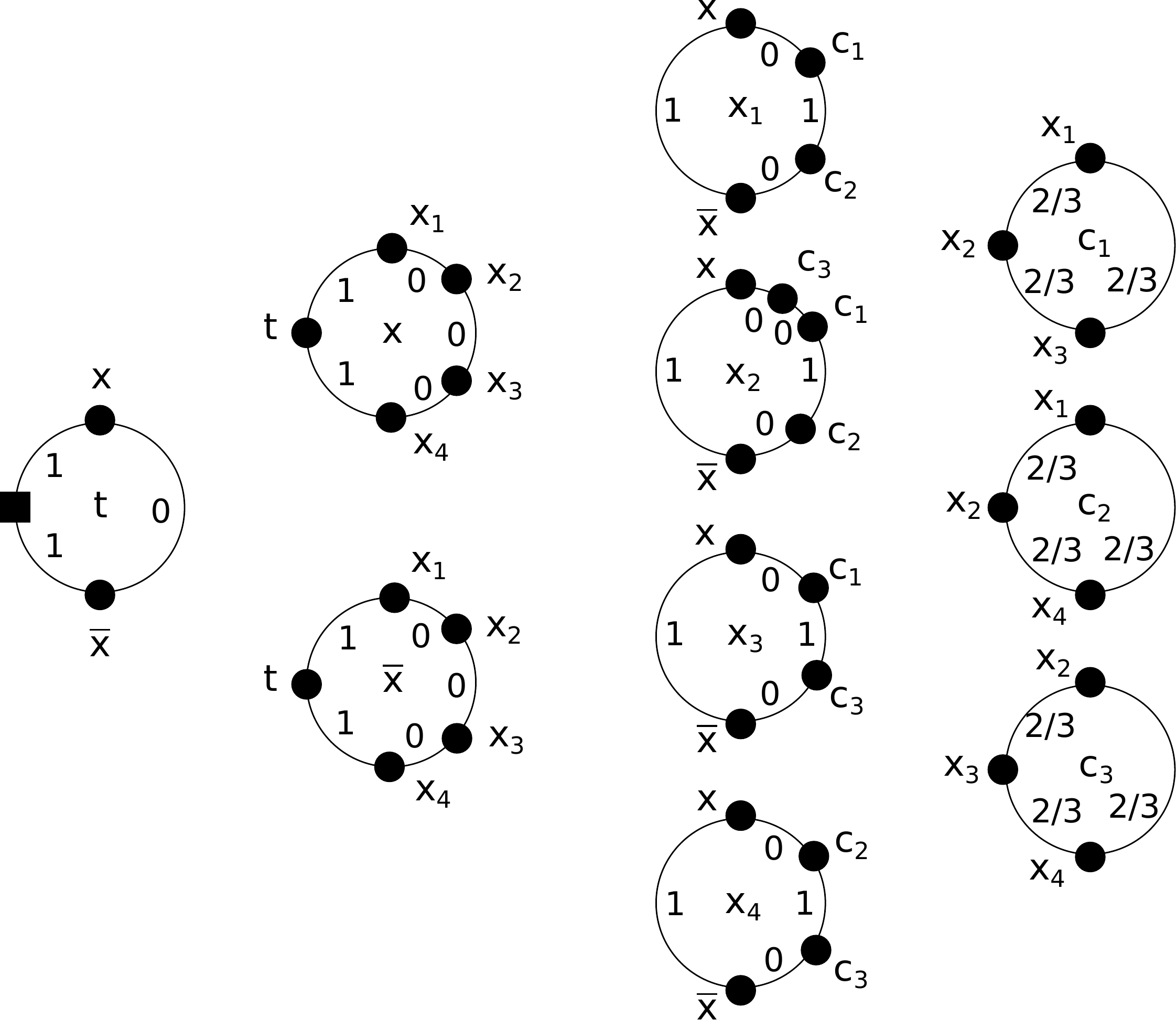}
	\caption{Example of a reduction from the 3SAT instance to d-MDTD (same example as in Figure~\ref{fig:example_3sat_mdt})}
	\label{fig:example_3sat_mdtd}
\end{figure}

\subsection{Selecting unique meeting points}
\label{subsec:mdtdm}

In case there are more than one potential meetings points between two tours, a unique set of meeting points has to be selected to obtain a tour graph (without multiple edges between two vertices). The decision problem d-MDTDM (minimum delay tree with directions and meeting points) is also NP-hard:

\begin{proposition}
d-MDTDM is NP-hard.
\end{proposition}
\begin{proof}
The proof is based on a similar idea as the proofs of Proposition~\ref{prop:mdt_np} and Proposition~\ref{prop:mdtd_np}. Figure~\ref{fig:example_3sat_mdtdm} shows the reduction of the 3SAT instance of the example in Figure~\ref{fig:example_3sat_mdt}. An assignment of 3SAT selects the meeting point between $x'_i$ and $x_i$: if $x_i=True$, the upper meeting point is selected (the directions of $x'_i$ and $x_i$ are counterclockwise), if $x_i=False$, the lower meeting point is selected (the directions of $x'_i$ and $x_i$ are clockwise). A satisfying assignment of 3SAT results in a delay of $17/2$. Selecting a meeting point between tours $x'_i$ and $x_i$ and their directions such that the delay is $17/2$, also determines a satisfying assignment for the 3SAT instance.
\end{proof}

\begin{figure}
	\centering
	\includegraphics[scale=0.29]{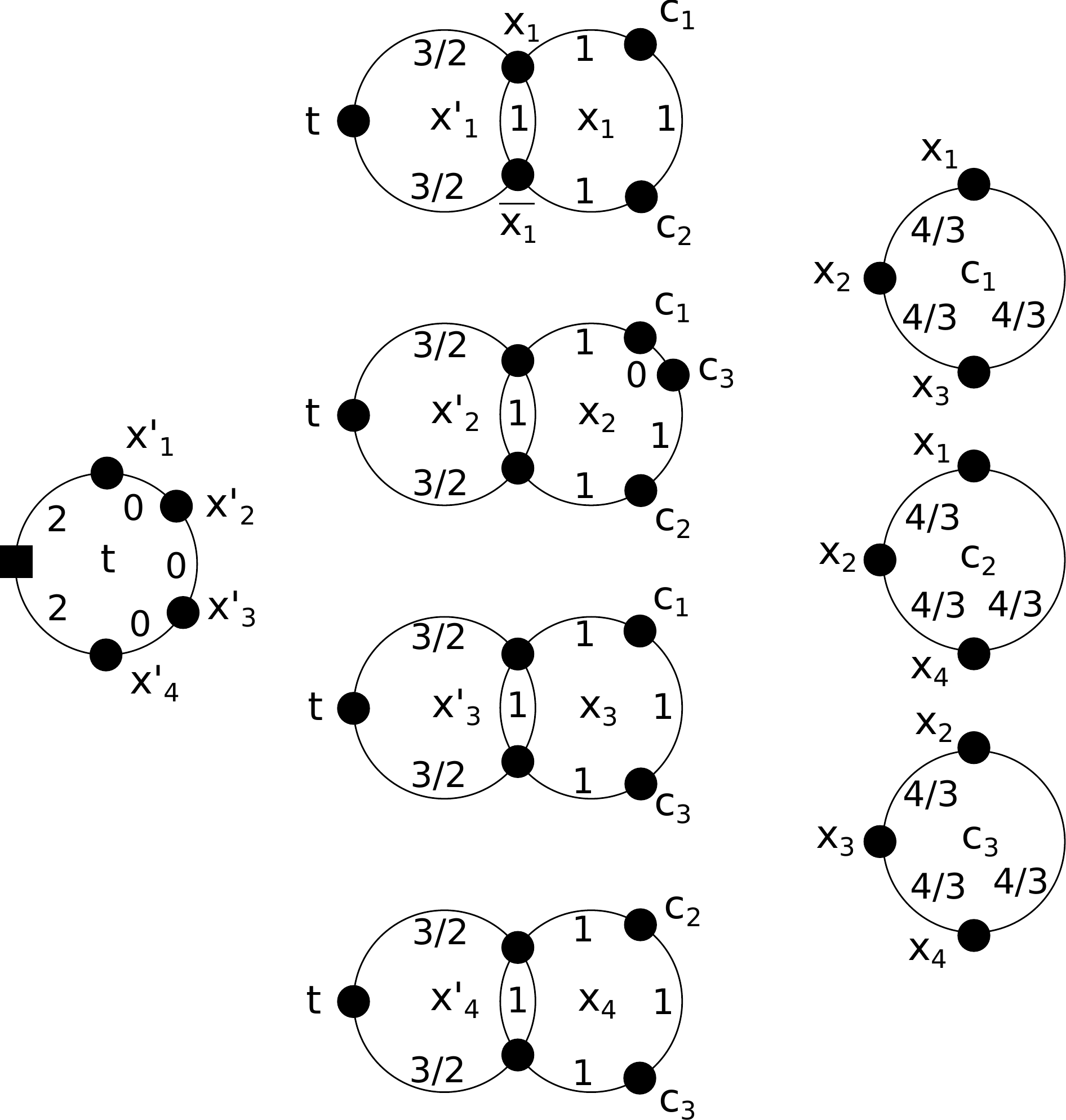}
	\caption{Example of a reduction from the 3SAT instance (same example as in Figure~\ref{fig:example_3sat_mdt}) to the problem of selecting unique meeting points between tours d-MDTDM.}
	\label{fig:example_3sat_mdtdm}
\end{figure}

\subsection{Approximation}

We have shown that the MDTD is already NP-hard when $P(v)=P_S(v)$ for all $v \in V$ and all tours have the same length. The formal definition allows tours containing (arbitrarily large) segments without sensing locations which can be used to derive the result that the problem can not be approximated with a constant factor unless $P=NP$. In Figure~\ref{fig:example_3sat_dg} a direction gadget is shown that is inserted between $x_i$ and $c_j$ if there is an edge in the tour graph (see Figure~\ref{fig:example_3sat_mdtd}). This gadget allows the data to pass in one direction within a delay of 2 but causes a delay of at least $\Gamma$ in the other direction and should prevent that the data from $c_j$ travels along a path on tours $c_j, x_i, c_k$. Additionally, the segments on $x_i$ which are 1 in Figure~\ref{fig:example_3sat_mdtd} get $\Gamma$ and do not contain sensing locations. Then, as before a $WI$ of 6 gives also a solution of the 3SAT instance. Now for every $\alpha$, $\Gamma$ is chosen large enough, e.g. $\Gamma=7 \alpha$, and an $\alpha$-approximation also results in a solution for the 3SAT instance.

A straightforward approximation for the case $P(v)=P_S(v)$ is a breadth first traversal of the tour graph to determine a tour tree which is the union of the shortest paths from each vertex to the base station vertex. Since $L=\max_{v\in V}\{l_v\}$ is a lower bound for the optimal worst delay $WD_{OPT}$, the worst delay $WD_{SP}$ of a breadth first traversal starting from the base station tour cannot be worse than $depth_{SP}(G) \cdot WD_{OPT}$, where $depth_{SP}(G):=\max_{v \in V}\{dist_G(v, v_0)\}$ is the maximum length of all shortest paths in the (unweighted) tour graph from the tours to the base station tour, e.g. $depth_{SP}(G)=3$ which is the length of the path from tour 2 to tour 5 for the example in Figure~\ref{fig:example_simple_sched}.

\begin{figure}
	\centering
	\includegraphics[scale=0.35]{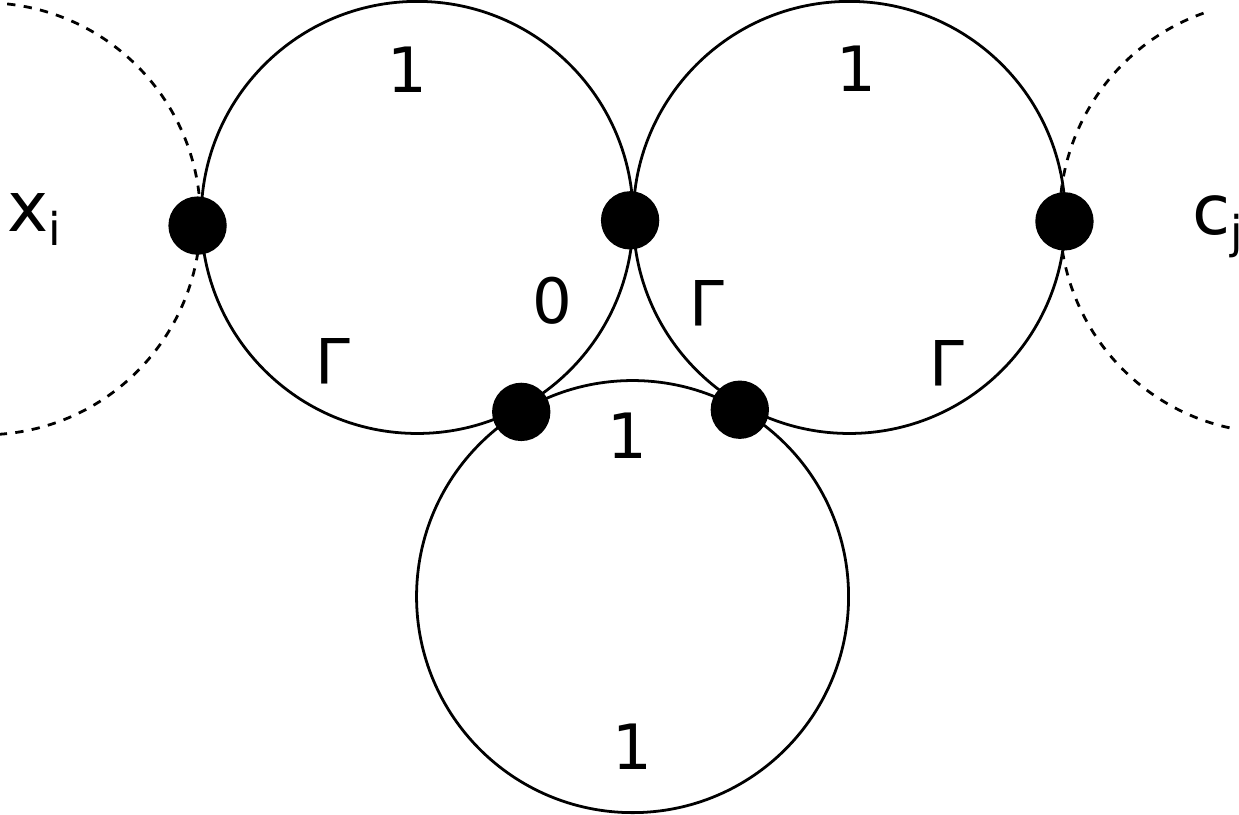}
	\caption{Direction gadget for d-MDTD.}
	\label{fig:example_3sat_dg}
\end{figure}

\section{Heuristics for MDTD}
\label{sec:mdtd_heur}

We present two heuristics for MDTD that select a tree in a tour graph and the directions for the tours. The first algorithm (\mbox{MDTD-SP}) determines a tree from the union of the shortest paths from all vertices to the base station vertex. The rationale behind this idea is to minimize the longest path in the tour graph in terms of the number of tours the generated data passes. This is shown in Algorithm~\ref{alg:mdtd_sp}.

\begin{algorithm}
	\caption{Heuristic for MDTD (MDTD-SP)}
	\label{alg:mdtd_sp}
	\small
	\begin{algorithmic}[1] % The number tells where the line numbering should start
		\Require
			\Statex \begin{flushleft} Tour graph $G=(V, E, v_0, l_v, time_v, l_v^d)$\end{flushleft}%to avoid spacing %base station vertex $v_0$ 
		\Ensure
			\Statex \begin{flushleft} Tour tree $T=(V,A)$, directions $d_v$ $\forall~v\in~V$ \end{flushleft}%to avoid spacing 
		
		\State $p \gets shortest\_path_{G}(v_0)$ /* single source Dijkstra */
		\State $A \gets \cup_{v\in V} \: p_v$
		\State /* Call Algorithm~\ref{alg:mindelay_sched} (returns directions $d_v$ $\forall~v\in~V$): */
		\State $(d_1, \ldots, d_n) \gets \Call{minimum\_delay\_schedule}{G, A}$
	\end{algorithmic}
\end{algorithm}

The second algorithm (\mbox{MDTD-CG}) is shown in Algorithm~\ref{alg:mdtd_cg} and requires a converted graph $G'=(V',E',W)$ with edge lengths $W$ which is constructed from a tour graph $G=(V,E)$. The vertices $V'$ of the converted graph contain the meeting points, i.e., if there is an edge $[k,l]\in E$, then there is a vertex $v_{kl}\in V'$. The length $W$ of the edges $E'$ between vertices in $V'$ are the lengths of the segments of the tours in $V$. An example of a tour graph and its converted graph is shown in Figure~\ref{fig:example_graph}. The idea behind this algorithm is to minimize the longest path that data actually travels on a path to the base station.

The algorithm determines the shortest path from every vertex in $v_{kl}$ (representing a meeting point between tours $k$ and $l$) to the base station $v_{0x}$. The function $dist_{G}(s,d)$ returns the length of the shortest path from vertex $s$ to vertex $d$ in a weighted graph $G$. This path represents a path for the data in the original tour graph for both tours $k$ and $l$ (in $path_k$ and $path_l$) and is stored together with the length of the path in $G'$ (in $len_k$ and $len_l$) if it is shorter than the shortest paths that have already been found for tours $k$ and $l$ (see the loop starting at Line~\ref{line:mdtd_cg_pathfor}). After this, the shortest paths in $G'$ for every tour $v \in V$ have been found. Note that the largest sum of the shortest path plus the tour length $\max_{v\in V}{(len_v + l_v)}$ is a lower bound on the worst delay $WD$.

Next, the branches of the tree $T$ are added to $A$ (loop in Line~\ref{line:mdtd_cg_treefor}). This is shown in Figure~\ref{fig:example_mdtd_heur} by means of the example of Figure~\ref{fig:example_graph}. Assume the longest path from any tour in $V$ starts at vertex $27$ (Figure~\ref{fig:example_mdtd_heur_1_converted}). This path determines a path $(2,7,3,6,5)$ in $G$ (which is added to the tree $T$) and the directions $d_7=d_6=cw, d_3=d_5=ccw$ (Figure~\ref{fig:example_mdtd_heur_1_graph}). In the next step the path starting at $14$ is considered (Figure~\ref{fig:example_mdtd_heur_2_converted}). This path would result in the path $(1,4,3,5)$ in $G$. Since $3$ is already part of the tree, only the branch $(1,4,3)$ is added to the tree (Figure~\ref{fig:example_mdtd_heur_2_graph}), and $d_4=ccw$. All tours are part of the tree $T$ and the algorithm stops. The directions of the leaves $1$ and $2$ are set according the rule for leaves in Algorithm~\ref{alg:mindelay_sched}.

Basically, in Line~\ref{line:mdtd_cg_checklefttour} the algorithm checks if a tour has been left and adds the appropriate arc to the arc set $U$. If a tour is already in the tree $T$, the path loop exits (Line~\ref{line:mdtd_cg_checkexitloop}), and the algorithm continues with the next tour in $V$. The direction of the tour $m$ which has been left depends on the order of meeting points $v_{rs}$ and its successor on $path_i$ on the tour $m$ (Line~\ref{line:mdtd_cg_dir}).

\begin{figure}
	\centering
	\begin{tabular}{cc}
	\subfloat[]{
		\includegraphics[scale=0.35]{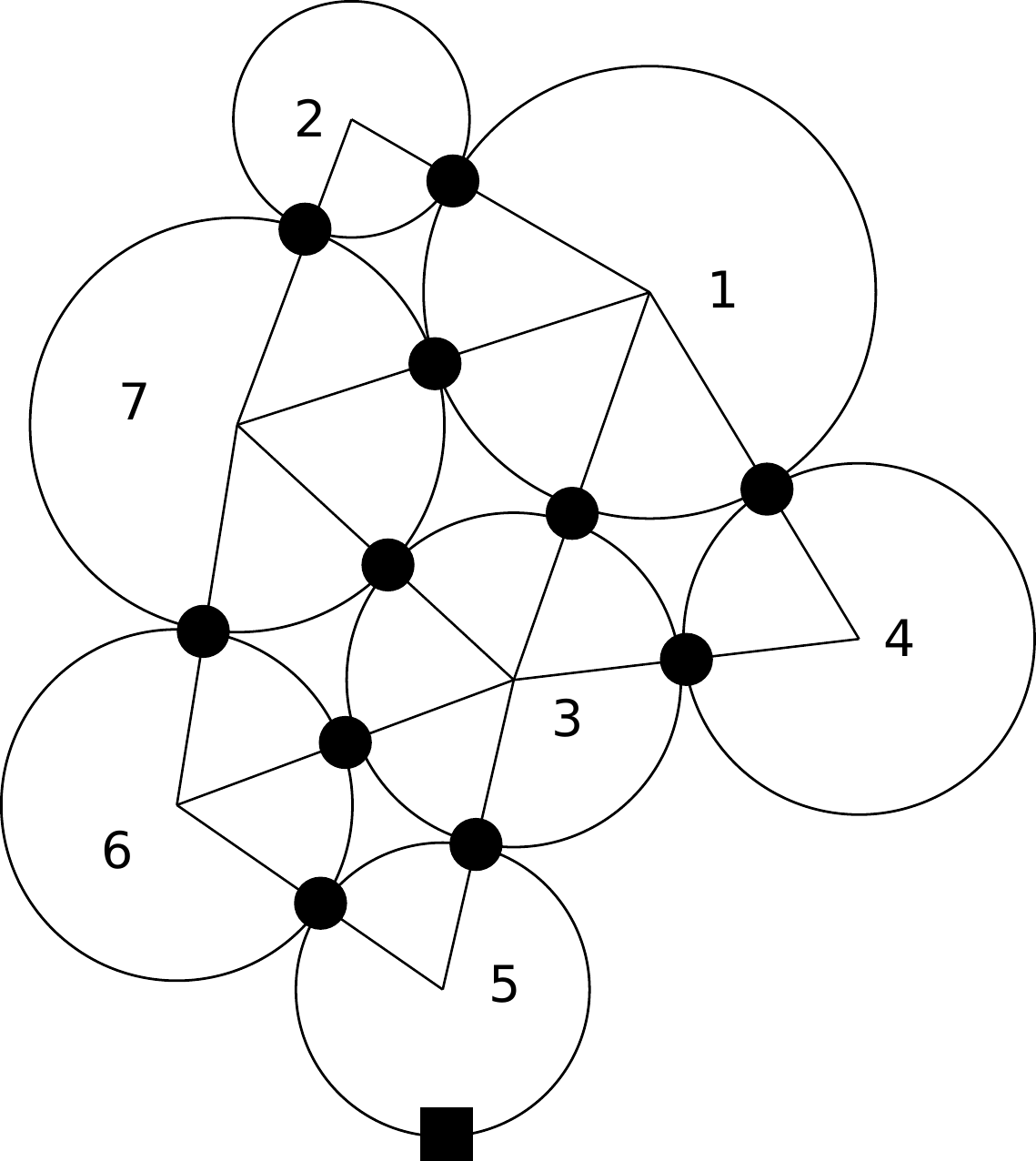}
		\label{fig:example_graph_overview}
	}
	&
	\subfloat[]{
		\includegraphics[scale=0.36]{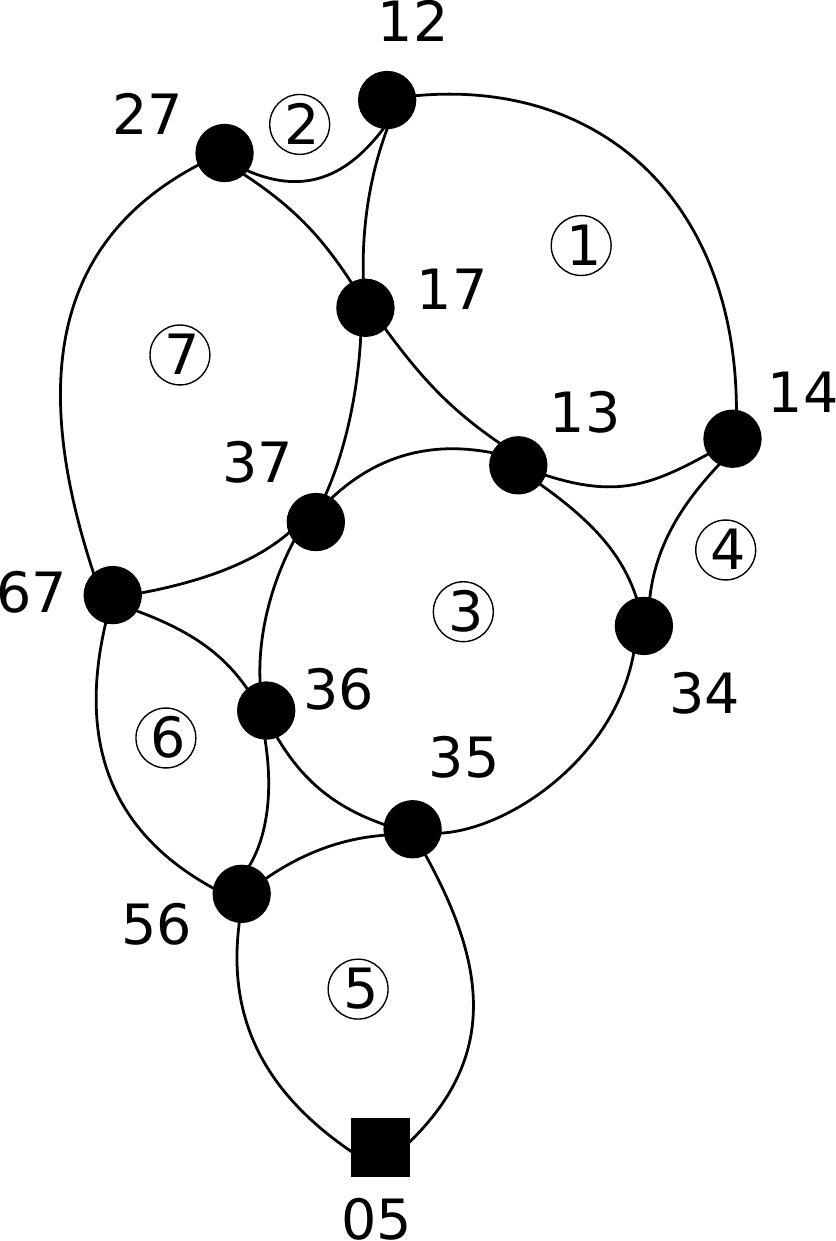}
		\label{fig:example_graph_converted}
	}
	\end{tabular}
	\caption{Example of a tour graph $G=(V,E)$ (a) and the converted graph $G'=(V',E',W)$ (b). The meeting points of the original graph are the vertices $V'$ in the new graph and the lengths $w$ of the edges in $E'$ between the vertices in $V'$ have the minimum travel times of the segments of the original tours. If there are two edges between two meeting points, the longer one is discarded.}
	\label{fig:example_graph}
\end{figure}

\begin{algorithm}
	\caption{Heuristic for MDTD (MDTD-CG)}
	\label{alg:mdtd_cg}
	\small
	\begin{algorithmic}[1] % The number tells where the line numbering should start
		\Require
			\Statex \begin{flushleft} Tour graph $G=(V,E)$, converted graph $G'=(V',E',W)$ with edge lengths $W(e)$, $\forall~e~\in~E'$, base station $v_{0x}$ \end{flushleft}%to avoid spacing 
		\Ensure
			\Statex \begin{flushleft} Tour tree $T=(V,A)$, directions $d_v$ $\forall~v\in~V$ \end{flushleft}%to avoid spacing 
	\State $len_i \gets \infty$, $\forall i \in V$
	\State $p \gets shortest\_path_{G'}(v_{0x})$
	\For{$v_{kl} \in V'$}\label{line:mdtd_cg_pathfor}
		%\If{$length(p) < len_k$}
		\If{$dist_{G'}(v_{kl}, v_{0x}) < len_k$}
			%\State $len_k \gets length(p)$
			\State $len_k \gets dist_{G'}(v_{kl}, v_{0x})$
			\State $path_k \gets p_{v_{kl}}$
		\EndIf
		%\If{$length(p) < len_l$}
		\If{$dist_{G'}(v_{kl}, v_{0x}) < len_l$}
			%\State $len_l \gets length(p)$
			\State $len_l \gets dist_{G'}(v_{kl}, v_{0x})$
			\State $path_l \gets p_{v_{kl}}$
		\EndIf
	\EndFor
	\State $A \gets \emptyset, r \gets null, s \gets null$
	\State $del_i \gets \min_{d \in \{cw, ccw\}}\{l_i - l_i^d(path_i(1), d)\}$
	\For{$i \in V \text{ in descending order of } (len_i+del_i)$}\label{line:mdtd_cg_treefor}
		%\State $v_{rs} \gets$ first element on $path_i$
		\State $U \gets \emptyset$, $m \gets i$
		\For{each $v_{kl}$ on path $path_i$ to $v_{0x}$}
			\If{$m$ is already in $T$}\State{break}\EndIf \label{line:mdtd_cg_checkexitloop}
			\If{$(m \neq k)$ and $(m \neq l)$}\label{line:mdtd_cg_checklefttour}
				\If{$m = r$}
					\State $U \gets U \cup \{(m, s)\}$, $m \gets s$
				\Else \If{$m = s$}
					\State $U \gets U \cup \{(m, r)\}$, $m \gets r$ \EndIf \EndIf
				\State $d_m \gets $ determine direction based on\label{line:mdtd_cg_dir}
				\StatexIndent[4] order of $v_{rs}$ and its successor on 
				\StatexIndent[4] $path_i$ on tour $m$
			\EndIf
			%\If{Cycle of form $m, i_1, \ldots, i_t, m$ is detected}
			%	\State Remove this cycle from $U$
			%\EndIf
			\State $v_{rs} \gets v_{kl}$
		\EndFor
		\State $A \gets A \cup U$
	\EndFor
	
	\end{algorithmic}
\end{algorithm}

\begin{figure}
	\centering
	\begin{tabular}{cc}
		\subfloat[]{
			\label{fig:example_mdtd_heur_1_converted}
			\includegraphics[scale=0.35]{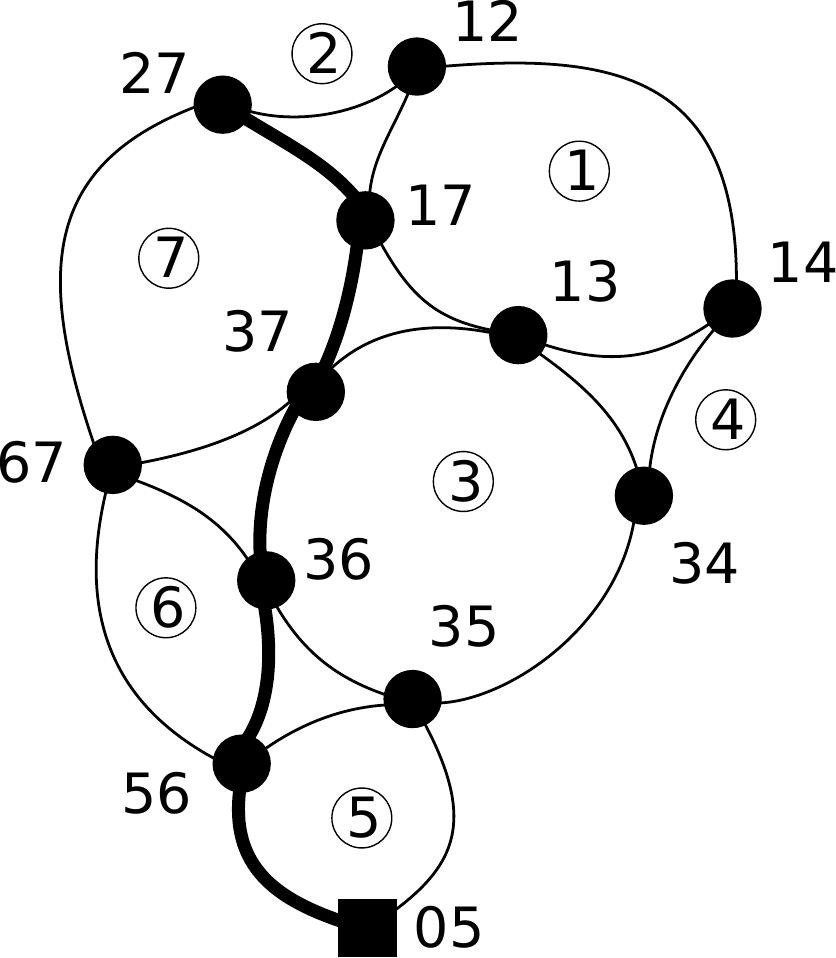}
		}
		&
		\subfloat[]{
			\label{fig:example_mdtd_heur_2_converted}
			\includegraphics[scale=0.35]{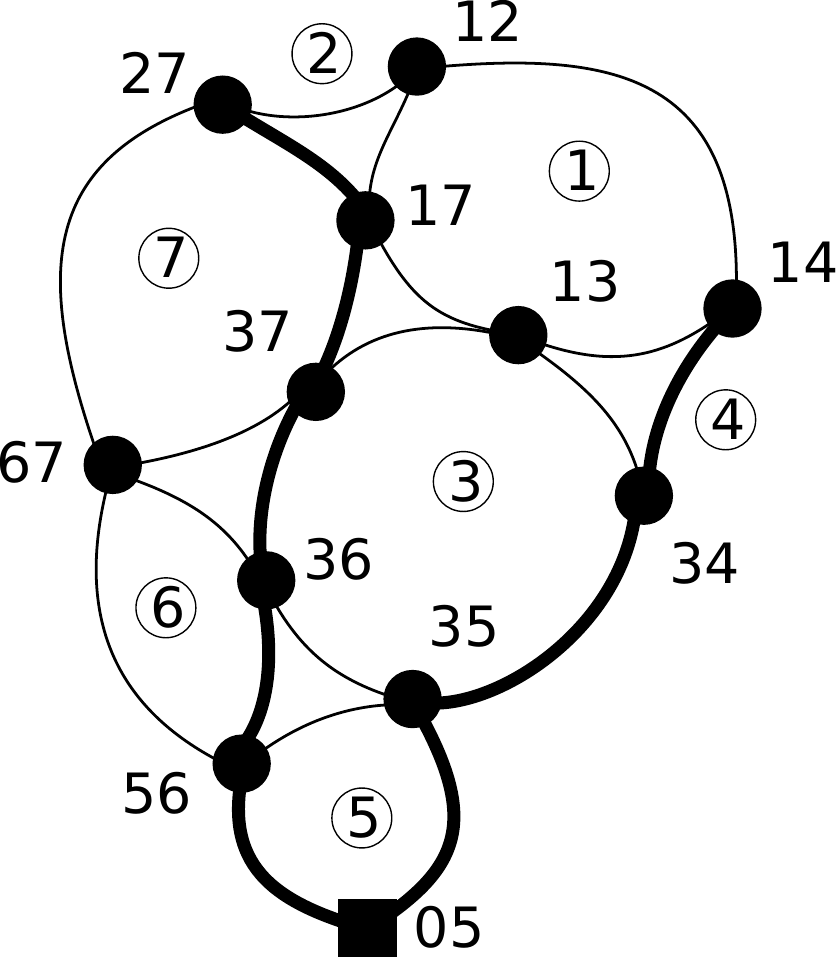}
		}
		\\
	%	\hfill
		\subfloat[]{
			\label{fig:example_mdtd_heur_1_graph}
			\includegraphics[scale=0.25]{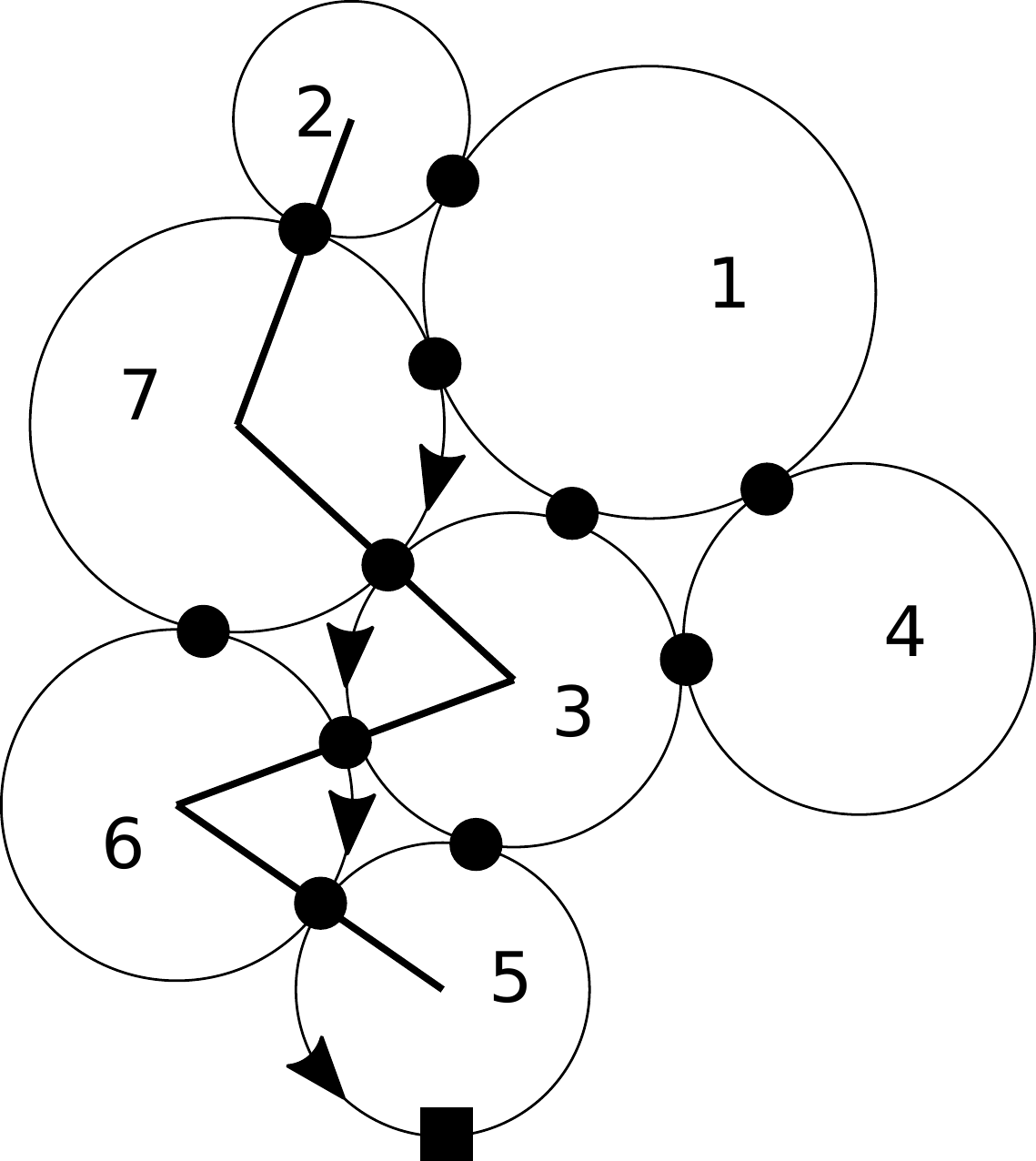}
		}
		&
		\subfloat[]{
			\label{fig:example_mdtd_heur_2_graph}
			\includegraphics[scale=0.25]{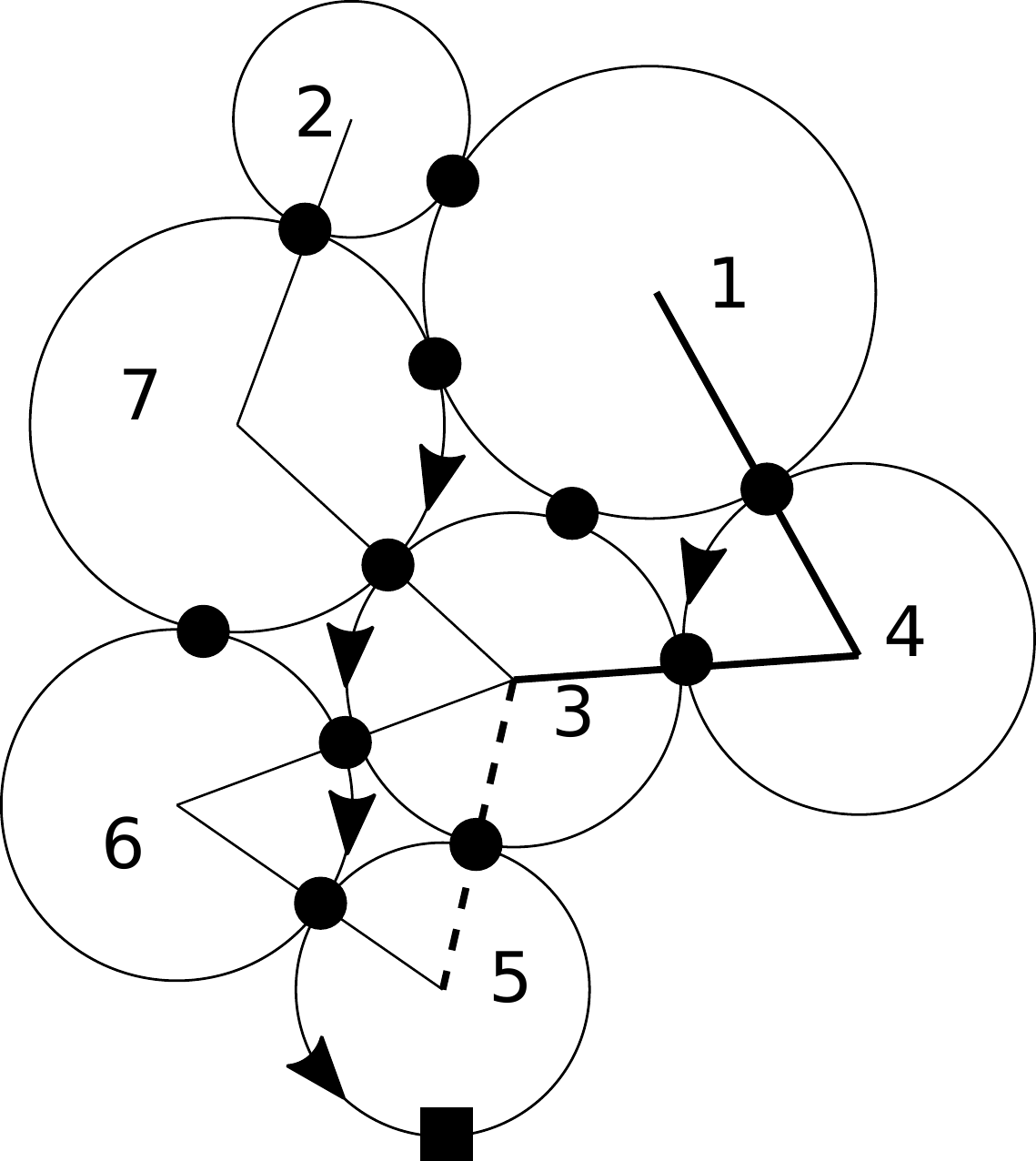}
		}
	\end{tabular}
	\caption{Two steps of the tree generation of Algorithm~\ref{alg:mdtd_cg}. First, the path in $G'$ starting from $27$ is considered (bold edges in Figure~\ref{fig:example_mdtd_heur_1_converted}), which results in the path $(2,7,3,6,5)$ in $G$ and the directions $d_7=d_6=cw, d_3=d_5=ccw$ (bold lines in Figure~\ref{fig:example_mdtd_heur_1_graph}). Next, the path starting at $14$ is considered (Figure~\ref{fig:example_mdtd_heur_2_converted}), which results in the path $(1,4,3,5)$ in G. Since $3$ is already in the tree, only the branch $(1,4,3)$ with $d_4=ccw$ is added (Figure~\ref{fig:example_mdtd_heur_1_graph}, the dashed line indicates the discarded part of the path). After this step the tree contains all tours.}
	\label{fig:example_mdtd_heur}
\end{figure}

\begin{proposition}
	Let $WD_{SP}$ and $WD_{CG}$ be the worst delay of a tree determined by \mbox{MDTD-SP} and \mbox{MDTD-CG}, respectively. Then, for every $\alpha>0$ there are instances of tour graphs such that $WD_{SP}/WD_{CG} > \alpha$.
\end{proposition}
\begin{proof}
	Consider a tour graph with a chain $[v_0, v_1], [v_1, v_2], \ldots, [v_{k-1}, v_k]$ of large tours of length $\Gamma$ where each tour is connected with an arm of small tours to $v_0$ (see Figure~\ref{fig:example_ratio_heur}). Each tour on an arm has length $\epsilon$ and each arm has at least $k$ tours. The meeting points on the chain of the large tours are on the opposite sides of the tours, i.e. $time_{v_i}(p_{v_i}^{meet}(v_{i-1}), p_{v_i}^{meet}(v_{i+1}))=\Gamma/2$. Then \mbox{MDTD-SP} will result in tree where all large tours are in a chain. If $\Gamma / \epsilon$ is large enough, \mbox{MDTD-CG} will create a tree where each large tour is connected with its arm to $v_0$. If $k > 2\alpha$, then $WD_{SP}/WD_{CG} > \alpha$.
\end{proof}

\begin{figure}
	\centering
	\includegraphics[scale=0.6]{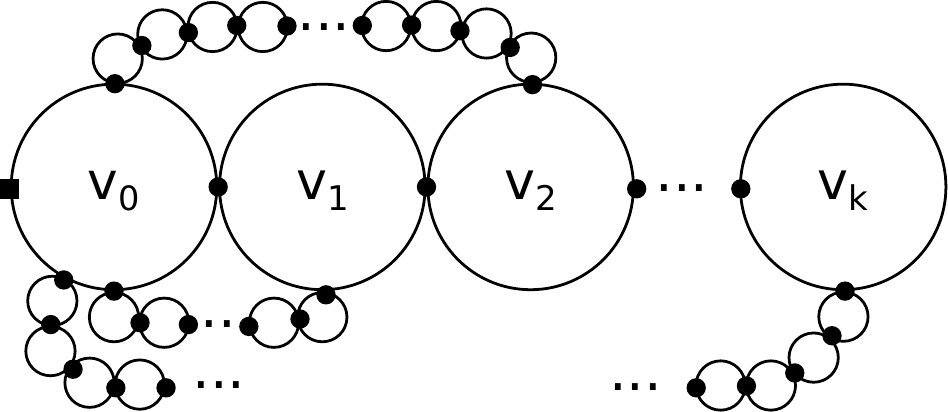}
	\caption{Example of a tour graph (size of the circle indicates the length $l_v$).}
	\label{fig:example_ratio_heur}
\end{figure}

\section{Online execution}
\label{sec:online}

Once the tree (Section~\ref{sec:mdtd_heur}), the directions and the schedule (Subsection~\ref{subsec:dir}) have been determined, the robots have to execute this schedule. If the robots need to be deployed in the environment the schedule determined by Algorithm~\ref{alg:mindelay_sched} has to be reached by the robots from an initial state. We assume that each robot navigates along some path in the environment to the meeting point with its parent at the beginning of the mission and reaches this position after some time. The algorithm for the online execution is described in the following subsection.

\subsection{State machine}

The algorithm running on every robot $v$ is shown in Algorithm~\ref{alg:online_sm} which resembles a state machine where the variable $state$ can take one of the states $\{INIT, AT\_WAIT, MOVING\}$. The robot is in $INIT$ state as long as it is moving from the initial position to the meeting point with its parent on its tour ($p_v^{start}$). In state $AT\_WAIT$ it is waiting for its parent on $p_v^{start}$, and in state $MOVING$ it is moving along its tour. The input to the algorithm is the schedule (in particular $\Delta_{uv}$ determined by Algorithm~\ref{alg:mindelay_sched}) and the output are commands $Move$ and $Stop$ for the motion actuators. The state of the state machine is initially $INIT$. We will show that under this assumption there is an infinite sequence of state transitions $INIT, WAIT\_AT, MOVING, WAIT\_AT, \ldots$ and that the schedule will converge to the optimal schedule after a finite time. A state transition $A$, $B$ means that variable $state$ changes from $state=A$ to $state=B$. Because of the assumption that every robot $v$ will reach $p_v^{start}$, a transition from $state=INIT$ to $state=AT\_WAIT$ always happens for every robot.

\begin{proposition}
\label{prop:sm_live}
A robot never has to wait for an infinite time. This implies an infinite sequence of state transitions $WAIT\_AT, MOVING, \ldots$ for each robot.
\end{proposition}
\begin{proof}
The situations when a robot has to wait infinitely long is when condition $p_u(t)=p_u^{meet}(v)$ (waiting for the parent) in state $WAIT\_AT$ never holds or when condition $p_w(t) = p_w^{meet}(v)$ (waiting for a child) in state $MOVING$ never holds. Since the meeting points define a tour tree $T=(V,A)$, it is sufficient to show that no robot has to wait for its parent infinitely long. After a robot has met its parent, it is waiting for a finite time (line~\ref{line:sm_waitchild}) and traverses its tour. The proof is by induction on the number of robots in the tree. In the base case only the robot $v$ which has the base station on its tour is in the tree and the condition $p_0(t)=p_0^{meet}(v)$ always holds (0 is the base station and the parent of $v$). In the inductive step a tour $v$ is added to the tree. Since its parent $u$ does not have to wait for its parent and starts its tour after a finite waiting time, also $v$ meets $u$ at $p_v^{start}$ (when $p_u(t)=p_u^{meet}(v)$).
\end{proof}

\begin{proposition}
\label{prop:sm_opt}
After a finite number of state transitions for each robot from $MOVING$ to $AT\_WAIT$ the schedule has converged to the schedule determined by Algorithm~\ref{alg:mindelay_sched}, i.e. $\Delta t$ stays 0.
\end{proposition}
\begin{proof}
Consider a tour $v$ with largest distance from the base station tour in the tour tree which has only leaves as children. After $v$ met its parent it starts traversing the tour and possibly has to wait for children to reach the meeting point. After the first traversal (state transition from $MOVING$ to $AT\_WAIT$) of the tour all children started their tour and had enough time to finish their tour and to reach the meeting position with $v$ on the second traversal of $v$. Therefore $\Delta t$ will be 0 for $v$ after the second traversal. The same holds for the parent $u$ of $v$ after an additional state transition from $MOVING$ to $AT\_WAIT$ of $u$. This argument can be repeated until the base station tour is reached.
\end{proof}

Figure~\ref{fig:example_online} shows the startup phase and execution of the state machine for the given example tour tree. After the schedule has emerged, robot $5$ arrives late at $p_v^{start}$ (short horizontal line at robot $5$ indicates no progress on its tour) and the schedule emerges again.

\begin{algorithm}
	\caption{State machine of robot $v$}
	\label{alg:online_sm}
	\small
	\begin{algorithmic}[1] % The number tells where the line numbering should start
		\Require
			\Statex \begin{flushleft} $p_v^{start}, p_v^{meet}(w)$ $\forall w$ with $(w,v)\in A, \Delta_{uv}$ for $(v,u) \in A$, initially $state=INIT$ \end{flushleft}%to avoid spacing 
		\Ensure
			\Statex \begin{flushleft} actuator commands $\{Move, Stop\}$\end{flushleft}%to avoid spacing 
		\Switch{$state$}
			\Case{$INIT$}
				\If{$p_v(t)=p_v^{start}$}
					\State $\Delta t \gets 0$
					\State $state \gets AT\_WAIT$
					\State $Stop$
				\Else
					\State $Move$ to $p_v^{start}$
				\EndIf
			\EndCase
			\Case{$AT\_WAIT$}
				\If{$p_u(t)=p_u^{meet}(v)$}
					\State $Wait$ for $\max\{\Delta_{uv}-\Delta t, 0\}$	\label{line:sm_waitparent}
					\State $\Delta t \gets 0$
					\State $state \gets MOVING$
					\State $Move$ on tour
				\EndIf
			\EndCase
			\Case{$MOVING$}
				\If{$p_v(t)=p_v^{meet}(w)$}
					\State $Stop$
					\State $time \gets current\_time()$
					\State $Wait$ until $p_w(t) = p_w^{meet}(v)$	\label{line:sm_waitchild}
					\State $\Delta t \gets \Delta t + (current\_time()-time)$
					\State $Move$ on tour
				\Else
					\If{$p_v(t)=p_v^{start}$}
						\State $state \gets AT\_WAIT$
						\State $Stop$
					\EndIf
				\EndIf
			\EndCase
		\EndSwitch
	\end{algorithmic}
\end{algorithm}

\begin{figure}
	\centering
	\subfloat[]{
		\includegraphics[scale=0.45]{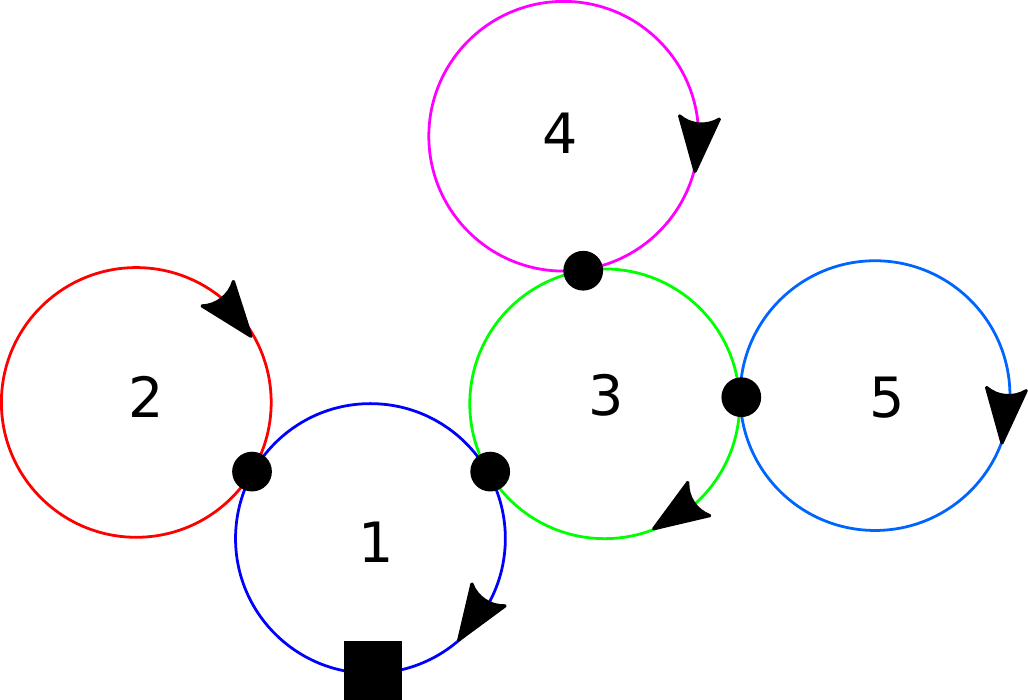}
		\label{fig:example_online_graph}
	}
	%elsearticle\\
	\subfloat[]{
		\includegraphics[scale=0.45]{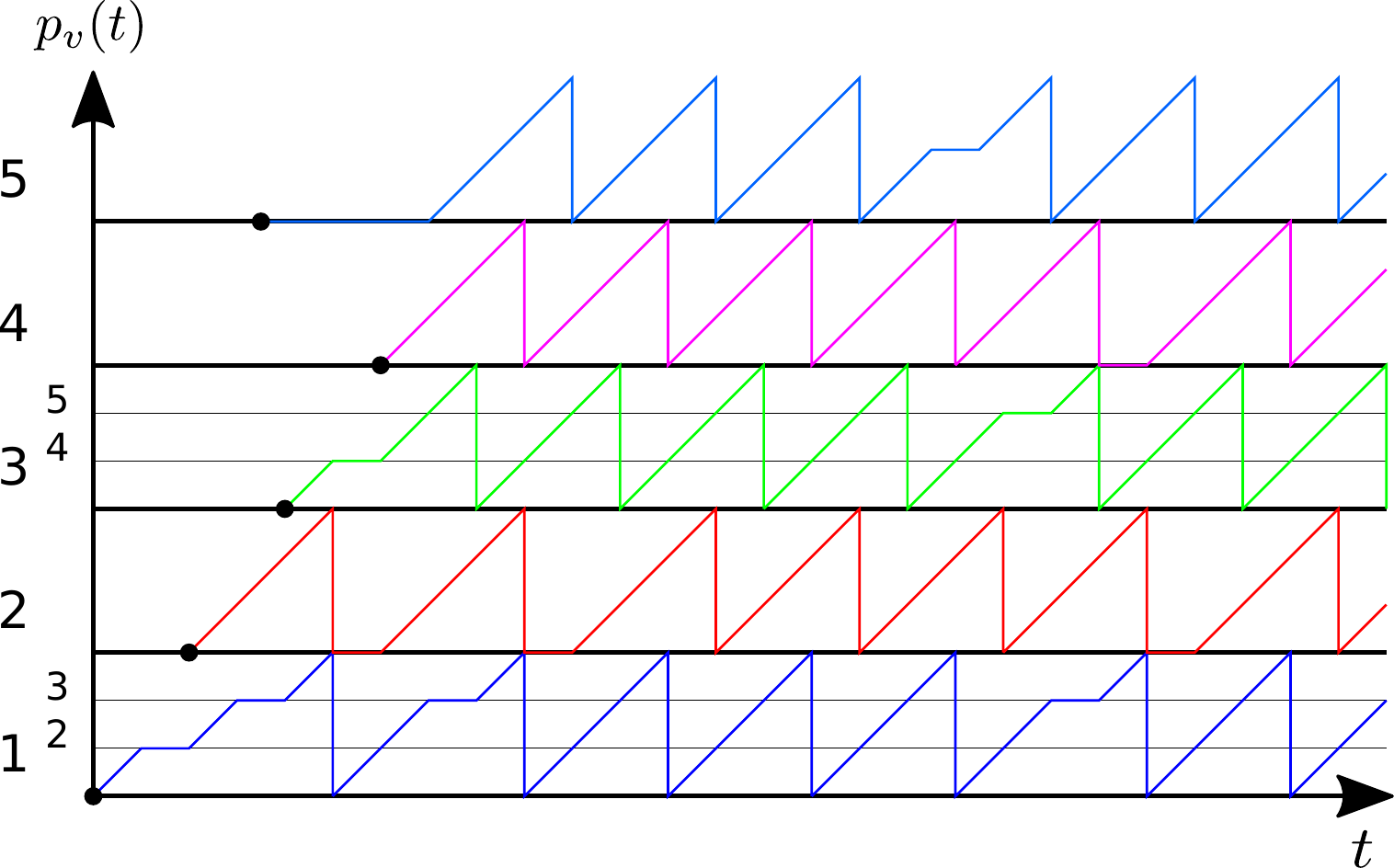}
		\label{fig:example_online_timing}
	}
	\caption{Example of startup phase and execution of the state machine. (a) Tour tree with tours of equal length and cw direction for all robots. (b) Position $p_v(t)$ over $t$ for robots 1 to 5 where the each bottom line indicates $p_v^{start}$. The small numbers on the vertical axis indicate the position of the meeting points on tour $v$. A small dot indicates when robot $v$ reaches $p_v^{start}$, i.e., the state transition from $INIT$ to $AT\_WAIT$.}
	\label{fig:example_online}
\end{figure}

\section{Experimental evaluation}
\label{sec:eval}

In this section we describe the results from simulation experiments with the aim to assess the performance of the heuristics (\mbox{MDTD-CG/SP}) in terms of worst idleness $WI$ and worst delay $WD$ in different situations (number of robots). To assess the effect of robots cooperating for the data transportation, we compare \mbox{MDTD-CG/SP} with an approach where the data is not transported via other robots to the base station but directly by the robot which captures the data (this approach is denoted as single-hop approach). Additionally, a breadth-first traversal of the tour graph has been implemented where the directions of the tours are determined by Algorithm~\ref{alg:mindelay_sched} from the resulting tree (\mbox{MDTD-SP}).

The environment is modeled as rectangular grid of cells of unit size, and time is discretized into time steps. A robot can move from one cell of the grid to one of the 8 neighboring cells or stay at the same cell within one time step. The communication range $R^{com}$ (measured in number of cells) determines which cells are within communication range. The base station is in the cell at the lower left corner.

A genetic algorithm implementation\footnote{Matlab function tsp\_ga from Joseph Kirk at https://www.mathworks.com/matlabcentral/ fileexchange/13680-traveling-salesman-problem-genetic-algorithm} is used to determine a tour through all sensing locations and the base station. To obtain the individual tours for the robots, the tour is split with \mbox{k-SPLITOUR} \cite{Frederickson1976}.

For \mbox{MDTD}, meeting points from a set of potential meeting points between each pair of tours have to be selected. In Section~\ref{subsec:mdt} we consider selecting the edges (corresponding to meeting points) in the tour graph such that the resulting graph is a tree, whereas here we are concerned with the selection of one of possible multiple edges between two vertices in the tour graph (cf. Section~\ref{subsec:mdtdm}). For a certain $R^{com}$, a potential meeting point between two tours is a pair of cells on the two tours within communication range. For the selection of meetings points, tours are traversed in a breadth-first order starting at the tour which is connected to the base station. The tours are added in the traversal order to a converted graph (the converted graph is described in Section~\ref{sec:mdtd_heur}), where the vertices are the meeting points selected so far. For every potential meeting point of $v$ with a neighboring tour $v'$, the shortest path to the base station on the converted graph is calculated, and the meeting point with the shortest path is selected as meeting point between $v$ and $v'$. This heuristic tries to shift meeting points as close as possible to the base station in the converted graph.

First, we compare the performance of \mbox{MDTD-CG/SP} with a single-hop algorithm similar to the one in \cite{Banfi2015} where robots make detours to communication sites to transmit the data. An increasing number of detours are inserted in a pre-calculated tour for each UAV until the total travel distance exceeds a certain travel budget. The heuristic of \cite{Banfi2015} tries to minimize the average delay and is not well suited for minimizing the worst delay. Here, an increasing number of detours to the base station, which are evenly spread along a robot's tour (using \mbox{k-SPLITOUR}), are inserted until a certain total tour length of a robot is exceeded. This bound for the tour length is set to the maximum of the worst idleness resulting from \mbox{MDTD-CG/SP} and the maximum tour length (including the base station) for each robot, such that every robot is able to transmit the data from its tour to the base station at least once. The results for $WI$ and $WD$ for different number of robots on a grid with an area of $20\times 60$ sensing locations is shown in Figure~\ref{fig:res_mdtd_shdetour_u_wi} and Figure~\ref{fig:res_mdtd_shdetour_u_wd}, respectively. Due to the stochastic nature of the genetic algorithm, the experiment is repeated 10 times for each $n$, and the standard deviation is a also shown in the figures.

Figure~\ref{fig:res_mdtd_shdetour_u_wd} also shows the worst delay $WD$ of the optimal solutions for the MDT instances generated with the state-of-the-art IP solver Gurobi\footnote{http://www.gurobi.com/} (an MDT instance is defined by the tours and the meeting points, see Appendix A for the MILP formulation). Note that the $WI$ is the same for \mbox{MDTD-CG/SP} and MDTD (opt).

From Figure~\ref{fig:res_mdtd_shdetour_u_wi} and Figure~\ref{fig:res_mdtd_shdetour_u_wd} it can be seen that on one hand \mbox{MDTD-CG}/\allowbreak SP/opt can outperform the single-hop approach in terms of $WI$. The value for $WI$ is the same for all three algorithms since all use the same tours. In the single-hop approach a robot from a more distant subarea has a long path to the base station, which causes a large $WI$. On the other hand, the single-hop approach can outperform \mbox{MDTD-CG/SP/opt} in terms of $WD$ because the data travels the shortest possible path to the base station which can be seen as a lower bound for $WD$ with given tours.

The average computation times for different number of robots is shown in Table~\ref{tab:soltime}. For \mbox{MDTD-CG/SP} a single core and for MDTD (opt) all 8 logical cores of a machine with an Intel Core-i7 6700K and 32GB of RAM were used. The instances for MDTD (opt) are the same as for Figure~\ref{fig:res_mdtd_shdetour_u_wi} and Figure~\ref{fig:res_mdtd_shdetour_u_wd}. The instances (tour graphs) for \mbox{MDTD-CG/SP} have been randomly generated (10 instances for each $n$), with edge probability of 0.25 between tours and randomly sampled meeting point distances on a tour.

In Figure~\ref{fig:res_mdtd_shdetour_u_dist} the sum of the traveled distances over all robots is shown for \mbox{MDTD-CG/SP} and the single-hop approach. The horizon for the delay calculation is the $WI$ achieved by \mbox{MDTD-CG/SP}. The reason is that for \mbox{MDTD-CG/SP} every sensing location gets visited once within this horizon. Since every robot is constantly moving in the single-hop approach the sum of the traveled distances is higher than for \mbox{MDTD-CG/SP} where robots might have to wait at meeting positions.

\begin{figure}
	\centering
	\begin{tikzpicture}
		\begin{axis}[
			height=4cm,
			scale only axis,
			width=0.7\columnwidth,			
			xmin=1, xmax=21,
			axis y line*=left,
			axis x line*=bottom,
			ymin = 0,
			xtick=data,
			xlabel=Number of robots $n$,
			ylabel=$WI$,
			legend style={at={(0.98, 1)}, font=\footnotesize, anchor=north east, legend columns=1, legend cell align=left}
			]

			%W=400 (starting at (6,6)), tour generated by tsp_ga(pop=1500, iter=15000), Rcom=1, dataex_full=false
			%MDTD-H1
			\addplot[mark options={scale=1,solid}, error bars/y dir=both, error bars/y explicit]
			table[x index=0, y index=1, y error index=2] {data/mdtd_shdetour_u.dat};	%(U, WI)

			%Single-hop-detour
			\addplot[mark=x, mark options={scale=1.5,solid}, error bars/y dir=both, error bars/y explicit]
			table[x index=0, y index=9, y error index=10] {data/mdtd_shdetour_u.dat};	%(U, WI)

			\legend{MDTD-CG/SP/opt\\Single-hop\\}
			
		\end{axis}

	\end{tikzpicture}
	\caption{Worst idleness $WI$ for \mbox{MDTD-CG/SP/opt} and Single-hop-detour with varying number of robots $n$, and neighboring cells are within communication range ($R^{com}=1$).}
	\label{fig:res_mdtd_shdetour_u_wi}
\end{figure}
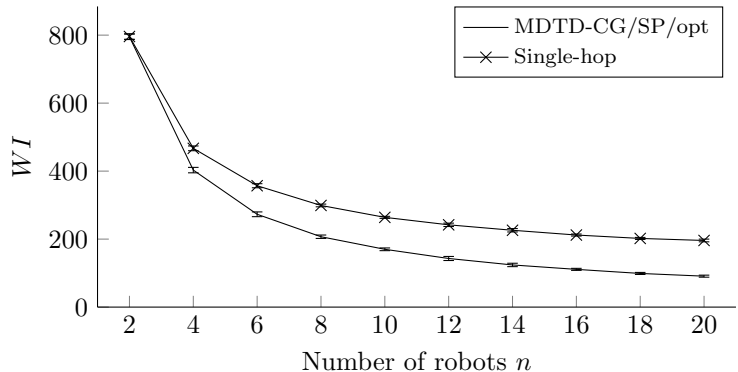

\begin{table}
	\centering
	\begin{tabular}{lrrrrrr}
		\hline
		$n$        & 50 & 100 & 150 & 200 & 250 & 300 \\
		\hline
		MDTD-SP    &  0 &   0 &  0  &  0  &   0 &  0 \\
		MDTD-CG    &  0 &   0 &  1  &  3  &   8 & 18 \\
		\hline
		\hline
		$n$        &  8 &  10 &  12 &  14  & 16 &   18 \\
		\hline
		MDTD (opt) &  8 &  28 & 111 & 565 &  1625 & 9704 \\
		\hline
	\end{tabular}
\caption{Average computation times (sec) for different number of robots $n$ for \mbox{MDTD-CG/SP/opt}.}
\label{tab:soltime}
\end{table}

There are situations for which the single-hop approach performs arbitrarily bad in terms of $WI$, e.g., the delay is unbounded if it is not possible for each robot to travel to the base station due to obstacles. Figure~\ref{fig:scenario_corridor} shows a scenario (20x40 cells) with predefined tours where the worst idleness and delay for \mbox{MDTD-CG} is 31 and 65, respectively, and for the single-hop approach 93 and 51, respectively (all tours have approximately the same length of 30 cells). The large worst idleness of the single-hop approach compared to \mbox{MDTD-CG} is obvious, since the robots, which traverse the right most tours, have long paths to the base station, whereas $WD$ is only slightly larger for \mbox{MDTD-CG}, since the data follows a path that is close to the shortest one to the base station.

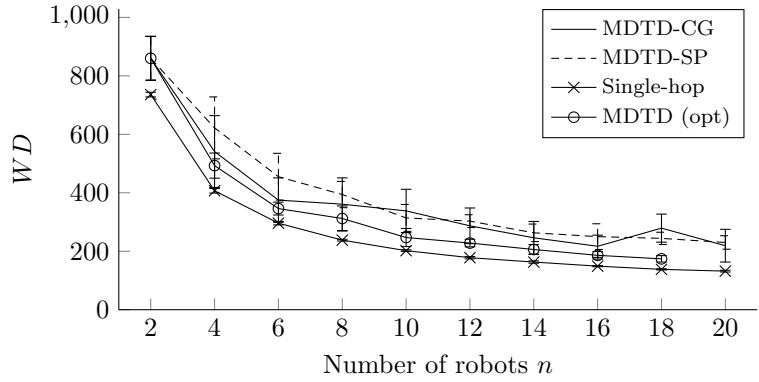
\begin{figure}
	\centering
	\begin{tikzpicture}
		\begin{axis}[
			height=4cm,
			width=0.7\columnwidth,			
			scale only axis,			
			xmin=1, xmax=21,
			ymin=0,
			axis y line*=left,
			axis x line*=bottom,			
			xtick=data,
			xlabel=Number of robots $n$,
			ylabel=$WD$,
			legend style={at={(0.98, 1)}, font=\footnotesize, anchor=north east, legend columns=1, legend cell align=left}
			]

			%W=400 (starting at (6,6)), tour generated by tsp_ga(pop=1500, iter=15000), Rcom=1
			%MDTD-H1
			\addplot[style=solid, mark options={scale=1.5,solid}, error bars/y dir=both, error bars/y explicit]
			table[x index=0, y index=3, y error index=4] {data/mdtd_shdetour_u.dat};	%(U, WD)
			
			%MDTD-SP
			\addplot[style=densely dashed, mark options={scale=1.5,solid}, error bars/y dir=both, error bars/y explicit]
			table[x index=0, y index=15, y error index=16] {data/mdtd_shdetour_u.dat};	%(U, WD)

			%Single-hop-detour
			\addplot[mark=x, style=solid, mark options={scale=1.5,solid}, error bars/y dir=both, error bars/y explicit]
			table[x index=0, y index=11, y error index=12] {data/mdtd_shdetour_u.dat};	%(U, WD)

			%MDTD-MILP
			\addplot[mark=o, style=solid, mark options={scale=1,solid}, error bars/y dir=both, error bars/y explicit]
			table[x index=0, y index=7, y error index=8] {data/mdtd_shdetour_u.dat};	%(U, WD)

			\legend{MDTD-CG\\MDTD-SP\\Single-hop\\MDTD (opt)\\}
			
		\end{axis}
	\end{tikzpicture}
	%\caption{Worst idleness $WI$, worst delay $WD$, and lowerbound on $WD$ ($WD$-$LB$) for \mbox{MDTD-H} with varying number of robots $n$, $R^{com}=1$.}
	\caption{Worst delay $WD$ for \mbox{MDTD-CG/SP/opt} and of the optimal solution (opt) for varying number of robots $n$, $R^{com}=1$.}
	\label{fig:res_mdtd_shdetour_u_wd}
\end{figure}

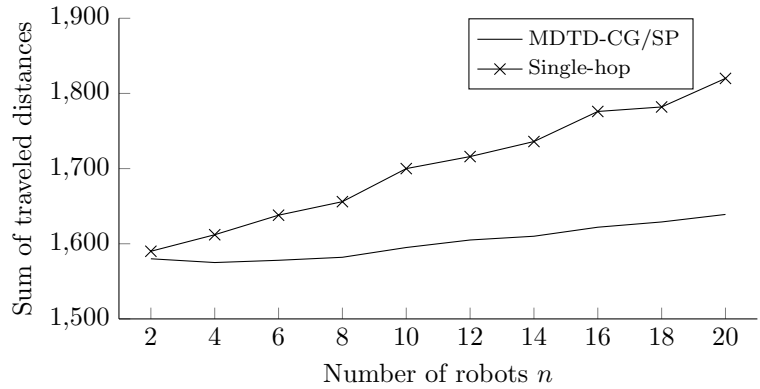
\begin{figure}
	\centering
	\begin{tikzpicture}
		\begin{axis}[
			height=4cm,
			scale only axis,
			width=0.7\columnwidth,			
			xmin=1, xmax=21,
			axis y line*=left,
			axis x line*=bottom,
			ymin = 1500, ymax = 1900,
			xtick=data,
			xlabel=Number of robots $n$,
			ylabel=Sum of traveled distances,
			legend style={at={(0.9, 1)}, font=\footnotesize, anchor=north east, legend columns=1, legend cell align=left}
			]

			%W=400 (starting at (6,6)), tour generated by tsp_ga(pop=1500, iter=15000), Rcom=1, dataex_full=false
			%MDTD-H1
			\addplot[mark options={scale=1,solid}]
			coordinates { %(U, dist)
				(2, 1580)
				(4, 1575)
				(6, 1578)
				(8, 1582)
				(10, 1595)
				(12, 1605)
				(14, 1610)
				(16, 1622)
				(18, 1629)
				(20, 1639)
				};

			%Single-hop-detour
			\addplot[mark=x, mark options={scale=1.5,solid}]
			coordinates { %(U, dist)
				(2, 1590)
				(4, 1612)
				(6, 1638)
				(8, 1656)
				(10, 1700)
				(12, 1716)
				(14, 1736)
				(16, 1776)
				(18, 1782)
				(20, 1820)
				};

			\legend{MDTD-CG/SP\\Single-hop\\}
			
		\end{axis}

	\end{tikzpicture}
	
	\caption{Sum of the traveled distances (number of steps in the grid) for \mbox{MDTD-CG/SP} and Single-hop-detour with varying number of robots $n$. The time within the distances have been calculated is $WI$ achieved by \mbox{MDTD-CG/SP}.}
	\label{fig:res_mdtd_shdetour_u_dist}
\end{figure}

\begin{figure}
	\centering
	\includegraphics[scale=0.3]{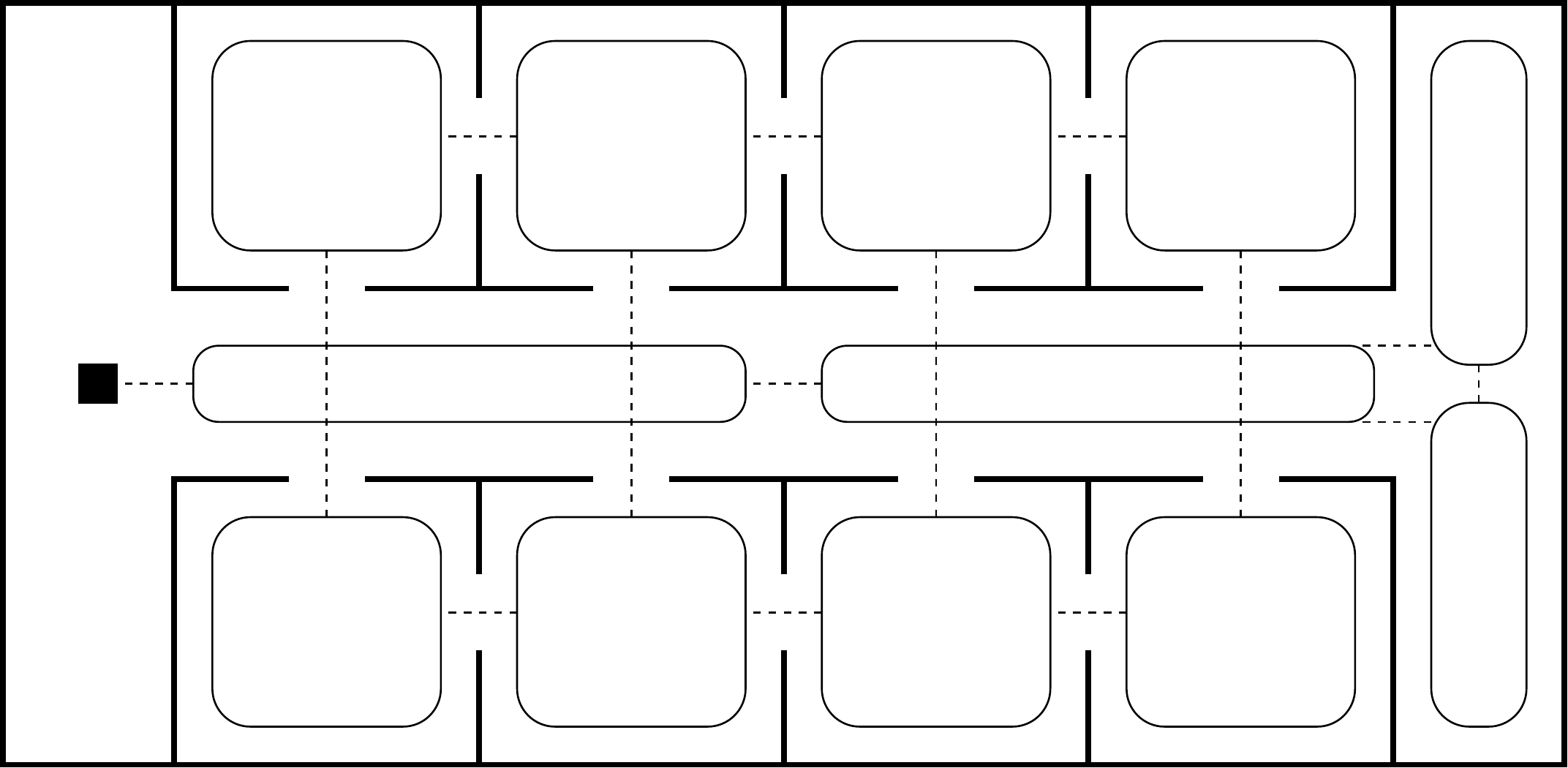}
	\caption{A scenario of size 20x40 cells for comparison of \mbox{MDTD-CG} with a single-hop approach. The bold lines show the border and obstacles of the environment (which prohibit movement and communication), the rectangles with rounded edges show the tours for the robots and the dashed lines show possible meeting points for \mbox{MDTD-CG}.}
	\label{fig:scenario_corridor}
\end{figure}

\section{Conclusion}
\label{sec:conclusion}

Multi-robot patrolling is an important application of multi-robot systems, and in certain situations it is not only important that sensing locations get visited repeatedly but also that the data reaches a base station on time for further processing or for an assessment by mission operators. This is typically required in disaster response scenarios where the mission operators need an up-to-date view of the situation. We presented a multi-robot patrolling problem with cooperative data transport to the base station where robots move on predefined tours which eliminates the need for every robot to return to the base station for data delivery. MDT represents the problem of minimizing the data delay which turns out to be NP-hard although its simple definition which is decoupled from path planning. Explicitly minimizing delay for patrolling with cooperative multi-robot data transport has not been investigated so far to the best of our knowledge. We presented heuristics and an algorithm for online execution and evaluated the performance in simulation experiments. The comparison of MDT with an uncooperative approach (every robot individually transports the data to the base station) on predefined tours shows that the cooperative approach can outperform the uncooperative approach in terms of WI and the traveled distances. The reason is, that the robots that handover the data to other robots (in this way the data finally reaches the base station) can continue patrolling their tours, while in the uncooperative approach, robots are forced to leave their tour to move to the base station.

The problem relies on TSP tours (and subtours derived from the TSP tours) through all sensing locations. In our work they are generated with traditional algorithms that try to minimize the length of the tour (and minimize the maximum length of the subtours). An open issue is the generation of such tours that support the joint minimization of idleness and delay. Other open questions are whether there are approximation algorithms with guaranteed bounds and whether some instance classes (e.g. planar graphs) can be solved optimally in polynomial time.

%\section*{Acknowledgment}

%elsarticle\ifCLASSOPTIONcaptionsoff
%elsarticle  \newpage
%elsarticle\fi

\section*{Appendix}

\subsection{MILP formulation of MDTD}
The mixed integer linear programming (MILP) model of MDTD is based on a multi-commodity flow formulation for trees on a graph $G=(V, A)$ with $n+1$ vertices $V$ (including a vertex 0 for a virtual base station tour) and arc set $A$ \cite{Magnanti1995}. The base station is the source of a commodity flow $f^c_e$ for each vertex (constraint~(\ref{eq:mdtdip_flowbs})). A flow of commodity $c$ represents the path of the data from robot $c$ towards the base station (though the flow originates at the base station in this formulation). For each vertex the sum of incoming flows is equal to the sum of outgoing flows for each commodity not dedicated to that vertex (constraint~(\ref{eq:mdtdip_flowvc})), and each vertex $c$ consumes the commodity of type $c$ (constraint~(\ref{eq:mdtdip_flowcc})). There can be only a flow on an edge if this edge is selected in the tree (constraint~(\ref{eq:mdtdip_fy})) and the sum of the edges must be $n$ (constraint~(\ref{eq:mdtdip_sumy})).

\begin{align}
	\sum_{(v,0) \in A}{f_{v0}^c} - \sum_{(0,v) \in A}{f_{0v}^c} &= -1 \quad \forall c \in V \setminus \{0\} \label{eq:mdtdip_flowbs} \\
	\sum_{(w,v) \in A}{f_{wv}^c} - \sum_{(v,w) \in A}{f_{vw}^c} &= 0 \quad \forall v \in V \setminus \{0, c\}, \forall c \in V \label{eq:mdtdip_flowvc} \\
	\sum_{(w,c) \in A}{f_{wc}^c} - \sum_{(c,w) \in A}{f_{cw}^c} &= 1 \quad \forall c \in V \setminus \{0\} \label{eq:mdtdip_flowcc} \\
	f_e^c &\leq x_e \quad \forall e \in A, \forall c \in V \setminus \{0\} \label{eq:mdtdip_fy} \\
	\sum_{e \in A}{x_e} &= n \label{eq:mdtdip_sumy} \\
	x_e &\in \{0, 1\} \\ 
	f_e^c &\geq 0 \quad \forall e \in A, \forall c \in V \setminus \{0\}
\end{align}

The data which robot $j$ gets at the meeting point between $i$ and $j$ and is forwarded at meeting point between $j$ and $k$ has to travel the distance $l_{ik}^{j,ccw}$ or $l_{ik}^{j,cw}$ on tour $j$, depending on the direction robot $j$ traverses its tour. Therefore, two flow variables $f_{ij}^c$ and $f_{jk}^c$ are involved in the cost calculation in constraint~(\ref{eq:mdtdip_objc}) for data originating from $c$ and traversing the tour $j$. The separation of the flows in this formulation allows the definition of a min-max objective. For each commodity $c$, $z_c$ models the delay of data originating at robot $c$ and the objective is to minimize $z$. The decision variables $u_j^{ccw}$ and $u_j^{cw}$ determine the direction robot $j$ traverses its tour.

\begin{align}
	z_c =& u_j^{ccw} \sum_{(j,c)\in A}{f_{jc}^c(l_c-l_c^d(p_c^{meet}(j), ccw))} + \\
	& u_j^{cw} \sum_{(j,c)\in A}{f_{jc}^c(l_c-l_c^d(p_c^{meet}(j), cw))} + \\
	& \sum_{(i,j),(j,k)\in A}{f_{ij}^c f_{jk}^c u_j^{ccw} l_{ik}^{j,ccw}} + f_{ij}^c f_{jk}^c u_j^{cw} l_{ik}^{j,cw} \quad \forall c \in V \setminus \{0\} \label{eq:mdtdip_objc} \\
	z_c &\leq z \quad \forall c \in V \setminus \{0\} \\
	u_j^{ccw} + u_j^{cw} &= 1 \quad \forall j \in V \setminus \{0\} \\
	u_j^{ccw}, u_j^{cw} &\in \{0, 1\} \quad \forall j \in V \setminus \{0\}
\end{align}

The products, e.g. $f_{ij}^c f_{jk}^c u_j^{ccw}$, can be linearized (likewise $f_{ij}^c f_{jk}^c u_j^{cw}$) with an additional variable $f_{ijk}^{c, ccw}$ and the constraints:

\begin{align}
%f_{ijk}^{c, ccw} &\leq f_{ij}^c \nonumber \\ &\begin{aligned} \forall (i, j), (j, k) \in A, c \in V \setminus \{0\} \end{aligned} \\
%f_{ijk}^{c, ccw} &\leq f_{jk}^c \nonumber \\ &\begin{aligned} \forall (i, j), (j, k) \in A, c \in V \setminus \{0\} \end{aligned} \\
%f_{ijk}^{c, ccw} &\leq u_j^{ccw} \nonumber \\ &\begin{aligned} \forall (i, j), (j, k) \in A, j, c \in V \setminus \{0\} \end{aligned} \\
%f_{ijk}^{c, ccw} &\geq f_{ij}^c + f_{jk}^c + u_j^{ccw} - 2 \nonumber \\ &\begin{aligned} \forall (i, j), (j, k) \in A, j, c \in V \setminus \{0\} \end{aligned}
f_{ijk}^{c, ccw} &\leq f_{ij}^c \\
f_{ijk}^{c, ccw} &\leq f_{jk}^c \\
f_{ijk}^{c, ccw} &\leq u_j^{ccw} \\
f_{ijk}^{c, ccw} &\geq f_{ij}^c + f_{jk}^c + u_j^{ccw} - 2
\end{align}

\subsection{List of symbols}

{\footnotesize
\begin{tabbing}
	\textbf{Symbol} \hspace{6em} \= \textbf{Meaning} \\
	$X$ \> set of points of environment \\
	$p_s$ \> sensing locations \\
	$Y$ \> communication relation \\
	$R=\{1, \ldots, n\}$ \> set of $n$ robots/tours \\
	$\pi \in \Pi$ \> patrolling strategy/schedule \\
	\> (from the set of all strategies $\Pi$) \\
	$\pi^+$ \> repeated schedule (repetition of schedule $\pi$) \\
	$\mathbb{R}_{\geq 0}$ \> set of real numbers larger or equal $0$ \\
	$I_t^{\pi}(x)$ \> instantaneous idleness of $x$ at time $t$ (using $\pi$) \\
	$D_t^{\pi}(x, t', t'')$ \> instantaneous delay of $x$ at time $t$ (using $\pi$) \\
	$WI_t^{\pi}(x)$ \> instantaneous worst idleness at time $t$ (using $\pi$) \\
	$WD_t^{\pi}(x)$ \> instantaneous worst delay at time $t$ (using $\pi$) \\
	$WI, WD$ \> worst idleness, worst delay \\
	$G=(V, E)$ \> (tour) graph with vertex set $V$ and edge set $E$ \\
	$G=(V, A)$ \> (tour) graph with vertex set $V$ and arc set $A$ \\
	$T=(V, E)$ \> (tour) tree with vertex set $V$ and edge set $E$ \\
	$[v,w] \in E$ \> (undirected) edge between $v$ and $w$ \\
	$(v,w) \in A$ \> (directed) arc from $v$ to $w$ \\
	$G'=(V',E',W)$ \> converted graph of tour graph $G$ \\
	$v_{kl}$ \> vertex of converted tour graph \\
	$v_0$ or $0$ \> base station \\
	$l_v$ \> minimum traversal time (without stops) of tour $v$ \\
	$L$ \> $\max_{v\in V}\{l_v\}$ \\ %largest tour traversal time among all traversal times \\
	$d_v$ \> direction robot $v$ traverses its tour (cw or ccw) \\
	$p_r(t)$ \> position of robot $r$ at time $t$ \\
	$time_v(p,q,d)$ \> minimum travel time on tour $v$ from point $p$ to point $q$ \\
	\> with direction $d$ \\
	$p_v^{start}$ \> start position of robot $v$ on its tour $v$ \\
	$p_v^{meet}(w)$ \> meeting point of robot $v$ on tour $v$ with robot $w$ \\
	%\> with robot $w$ \\
	$wait_v(p)$ \> waiting time for robot $v$ on meeting point $p$ \\
	$\Delta_{wv}$ \> $wait_w(p_w^{start})-wait_v(p_{v}^{start})$ \\
	$P(v)$ \> set of all positions on tour $v$ \\
	$dist_G(s,d)$ \> length of shortest path between vertices $s$ and $d$ \\
	\> in (weighted) graph $G$
\end{tabbing}
}

\bibliographystyle{plain}
\bibliography{references}

%\end{IEEEbiography}

\end{document}